\documentclass[10pt,twocolumn]{article}

\usepackage[margin=1in]{geometry}
\usepackage{microtype}
\usepackage{graphicx}
\usepackage{subcaption}
\usepackage{booktabs} 
\usepackage{algorithm}
\usepackage{algorithmic}
\usepackage[numbers,sort&compress]{natbib}
\usepackage{hyperref}       

\usepackage{amsfonts}       
\usepackage{xcolor}         

\usepackage{lipsum}
\usepackage{epstopdf}
\usepackage{mathtools}        
\usepackage{bm}

\usepackage{amsmath}
\usepackage{amssymb}
\usepackage{amsthm}

\theoremstyle{plain}
\newtheorem{theorem}{Theorem}[section]

\newtheorem{lemma}[theorem]{Lemma}

\theoremstyle{definition}
\newtheorem{definition}[theorem]{Definition}
\newtheorem{assumption}[theorem]{Assumption}
\theoremstyle{remark}
\newtheorem{remark}[theorem]{Remark}

\newcommand\xbm{{\ensuremath{\bm{x}}}}
\newcommand\kbm{{\ensuremath{\bm{k}}}}
\newcommand\Kbm{{\ensuremath{\bm{K}}}}
\newcommand\fbm{{\ensuremath{\bm{f}}}}

\newcommand\ybm{{\ensuremath{\bm{y}}}}

\newcommand\Ibm{{\ensuremath{\bm{I}}}}

\newcommand\Pbm{{\ensuremath{\bm{P}}}}
\newcommand\pbm{{\ensuremath{\bm{p}}}}

\begin{document}

\title{Practical Efficient Global Optimization is No-regret}

\author{
Jingyi Wang\\
\small Lawrence Livermore National Laboratory, CA, USA\\
\and
Haowei Wang\\
\small National University of Singapore, Singapore
\and
Nai-Yuan Chiang\\
\small Lawrence Livermore National Laboratory, CA, USA
\and
Juliane Mueller\\
\small National Laboratory of the Rockies, CO, USA
\and
Tucker Hartland\\
\small Lawrence Livermore National Laboratory, CA, USA
\and
Cosmin G. Petra\\
\small Lawrence Livermore National Laboratory, CA, USA
}

\date{}

\maketitle

%

\newcommand{\Rbb}{\ensuremath{\mathbb{R} }}
\newcommand{\Pbb}{\ensuremath{\mathbb{P} }}
\newcommand{\Cbb}{\ensuremath{\mathbb{C} }}
\newcommand{\Ebb}{\ensuremath{\mathbb{E} }}
\newcommand{\Sbb}{\ensuremath{\mathbb{S} }}
\newcommand{\Vbb}{\ensuremath{\mathbb{V} }}
\newcommand{\Nbb}{\ensuremath{\mathbb{N} }}
\newcommand{\norm}[1]{\left\lVert {#1} \right\rVert}
\newcommand\epsbold{{\ensuremath{\boldsymbol{\epsilon}}}}
\newcommand{\frank}[1]{\textcolor{purple}{[#1]}}
\newcommand{\cosmin}[1]{\textcolor{blue}{[#1]}}
\newcommand{\ny}[1]{\textcolor{green}{[#1]}}

\begin{abstract}
  Efficient global optimization (EGO) is one of the most widely used noise-free Bayesian optimization algorithms. 
  It comprises the Gaussian process (GP) surrogate model and expected improvement (EI) acquisition function.  
  In practice, when EGO is applied, a scalar matrix of a small positive value (also called a nugget or jitter) is usually added to the covariance matrix of the deterministic GP to improve numerical stability. We refer to this EGO with a positive nugget as the practical EGO. Despite its wide adoption and empirical success, to date, cumulative regret bounds for practical EGO have yet to be established. 
  In this paper, we present for the first time the cumulative regret upper bound of practical EGO. In particular, we show that practical EGO has sublinear cumulative regret bounds and thus is a no-regret algorithm for commonly used kernels including the squared exponential (SE) and Mat\'{e}rn  kernels ($\nu>\frac{1}{2}$). Moreover, we analyze the effect of the nugget on the regret bound and discuss the theoretical implication on its choice. Numerical experiments are conducted to support and validate our findings.
\end{abstract}

\section{Introduction}\label{se:introduction}
Efficient global optimization (EGO) is a derivative-free optimization method that uses Gaussian process (GP) surrogate models to approximate and guide the optimization of black-box functions~\citep{lizotte2008,jones2001taxonomy,bosurvey2023}. 
Given no observation noise, it is equivalent to Bayesian optimization (BO) with the expected improvement (EI) acquisition function. 
EGO has seen enormous success in many applications including machine learning~\citep{wu2019hyperparameter}, robotics~\citep{calandra2016}, aerodynamic optimization design~\citep{jeong2005efficient}, etc, and has been extended to constrained BO~\citep{gardner2014}, combinatorial problems~\citep{Zaefferer2014}, and multiple surrogates~\citep{viana2013efficient}. 
In the classic form, EGO aims to solve the optimization problem 
\begin{equation} \label{eqn:opt-prob}
 \centering
  \begin{aligned}
	  &\underset{\substack{\xbm}\in C}{\text{minimize}} 
	  & & f(\xbm), \\
  \end{aligned}
\end{equation}
where $\xbm\in\Rbb^d$ is the decision variable, $C\subset \Rbb^d$ represents the bound constraints on $\xbm$, and $f:\Rbb^d\to \Rbb$ is the black-box objective function. 

 EGO iteratively   selects  the candidate sample for
the next observation using  the EI acquisition function~\citep{jones1998efficient,schonlau1998global}.
EI computes the conditional expectation of an improvement function 
leveraging both the posterior mean and variance of the GP in search of the next sample. 
With a closed form, EGO is simple to implement and only requires the cumulative distribution function (CDF) and probability density function (PDF) of the standard normal distribution. 

In practice, adding a small positive value, called a \textit{nugget} or a \textit{jitter}~\citep{gramacy2012cases,gardner2018gpytorch}, to the diagonal of the GP covariance matrix in EGO is often beneficial or even necessary \citep{balandat2020botorch} (see Section~\ref{se:gpbackground} for details). 
We denote the nugget as $\epsilon>0$ in this paper. 
One of the main motivations for using $\epsilon$ is to improve the numerical stability of computations involving the inverse of the covariance matrix, \textit{e.g.}, calculating posterior mean and variance.
The covariance matrix is known to cause numerical issues due to ill-conditioning, which can occur when sample points are too close to each other~\citep{jones1998efficient}. 
Cholesky decomposition, which is invariably used to solve linear systems involving the covariance matrix, is known to break in finite-precision arithmetic for ill-conditioned matrices since very small (both positive and negative) pivots occur due to round-off errors \citep{wilkinson1968priori, meinguet1983refined, kielbasinski1987note, sun1992componentwise, lourenco2022exactly, higham2002accuracy}.
This numerical instability has long been recognized by the numerical optimization community and thoroughly investigated by proposing modified Cholesky decompositions employing positive semi-definite perturbations  \citep{gill1974newton, gill2019practical, schnabel1990new, wright1999modified}, or positive diagonal regularization terms  \citep{gondzio2012matrix}, to circumvent the factorization breakdown.

Indeed, the use of $\epsilon$ is widely adopted in some of the most popular BO/EGO software packages. In the Surrogate Modeling Toolbox (SMT)~\citep{saves2024smt}, the EGO implementation with the GP (KRG) model uses a fixed $\epsilon=2.220\times 10^{-14}$ with the option of larger values. 
In Pyro~\citep{bingham2018pyro}, the GP regression model (GPModel) employs $\epsilon=10^{-6}$ for stabilization of the Cholesky decomposition. 
In scikit-learn~\citep{scikit-learn}, the GP model (GPR) for EGO uses a default $\epsilon=10^{-10}$.
In Botorch~\citep{balandat2020botorch}, which uses GPyTorch~\citep{gardner2018gpytorch}, examples on GP regression models (SingleTaskGP) with noise-free observations are presented with $\epsilon=10^{-6}$. 
GPyOpt~\citep{gpyopt2016} uses a trial-and-error approach, where  $\epsilon=10^{-10}$ is added when the Cholesky decomposition of the covariance matrix fails.
BayesOpt~\citep{martinez2014bayesopt}, which is used in MATLAB, similarly raises errors when the Cholesky decomposition fails.
Finally, we mention that in GPyTorch~\citep{gardner2018gpytorch, wang2019exact}, a positive nugget is instrumental in the well-posedness and efficiency of the proposed preconditioned batched conjugate gradient algorithm.


Apart from improving numerical stability, in recent years, the adoption of $\epsilon$ for deterministic GP fitting and EGO has been recommended for improved statistical properties. 
In~\cite{andrianakis2012effect}, the authors studied the effect of the nugget  and showed that by choosing appropriate values for $\epsilon$ and the length-scale hyper-parameter of the squared exponential (SE) kernel, the approximation errors of the GP can be arbitrarily small.
In~\cite{gramacy2012cases}, the authors proposed adding  $\epsilon$ in deterministic GP models for noise-free observations. They noted that its inclusion provided improved statistical properties of the GP model for many common scenarios. 
Similarly,~\cite{pepelyshev2010role} recommended using $\epsilon$ for noise-free observations. The authors claimed that the maximum likelihood
estimate of the correlation parameter is more reliable and the condition number of the covariance matrix is moderate.
In~\cite{bostanabad2018leveraging}, the authors developed an adaptive strategy for choosing $\epsilon$ to improve the accuracy and efficiency in fitting the
hyper-parameters of a GP model.
Given its wide adoption in both literature and practical code implementation, we refer to EGO with a nugget $\epsilon>0$ as the \textit{``practical EGO''} in this paper. 

While the \textit{simple regret} bound of EGO has been studied in~\cite{bull2011convergence} in the frequentist setting, where $f$ lies in a reproducing kernel Hilbert space (RKHS), 
existing works on the \textit{cumulative regret} behavior of either EGO or practical EGO have clear limitations. 
Given $t$ samples $\xbm_1,\dots,\xbm_t$, simple regret measures the error between the smallest observed function value and the optimal function value, \textit{i.e.}, $f_t^+-f^*$, where  $f_t^+= 
 \underset{\substack{i=1,\dots,t}}{\text{min}}  f(\xbm_i)$  and $f^*=\underset{\substack{\xbm\in C}}{\text{min}}f(\xbm)$.
On the other hand, cumulative regret, denoted as $R_T$ for $T$ samples, measures the overall performance of the algorithm throughout the optimization process (see~\eqref{eqn:cumu-regret} for definition), and is the preferred metric in the multi-armed bandit paradigm~\citep{robbins1952some,lai1985asymptotically,agrawal2012analysis} and many real-world applications~\citep{bouneffouf2020survey}.   
It is desirable for an algorithm to have a sublinear $R_T$ asymptotically, \textit{i.e.}, $\lim_{T\to\infty} R_T/T=0$. This property is called the no-regret property.
Following the seminal work of~\cite{srinivas2009gaussian}, the cumulative regret bounds for some BO algorithms such as  the upper confidence bound (UCB) and Thompson sampling (TS) have been studied extensively, including in the noise-free case~\citep{chowdhury2017kernelized,vakili2022open,lyu2019efficient}. 

However, the cumulative regret upper bounds for either EGO or practical EGO have not been established, despite EI being one of the most popular acquisition functions~\citep{frazier2018}. 
From a technical perspective, this is partially due to the inclusion of an incumbent in EI and its non-convex nonlinear nature~\citep{ryzhov2016convergence,hu2022adjustedeiregret}.
Existing works often make noticeable modifications to the EI acquisition function by introducing new hyper-parameters in order to achieve sublinear cumulative regret bounds. 
In~\cite{wang2014theoreybo}, the authors studied a modified EI function with additional hyper-parameters and the best posterior mean incumbent. They further used the lower and upper bounds of the hyper-parameters to prove a sublinear cumulative regret bound. Similarly,~\cite{hu2022adjustedeiregret} introduced an evaluation cost and modified EI. 
In ~\cite{tran2022regret}, the authors also modified the EI 
by including additional control parameter and showed an upper bound for the sum of simple regret, not the cumulative regret.
In ~\cite{nguyen17a}, the authors added an additional stopping criterion to bound the instantaneous regret. However, it is unclear whether the stopping criterion guarantees an optimal solution upon exit. 


The lack of cumulative regret analysis despite the empirical success of practical EGO leaves two important open questions: \textit{Is practical EGO a no-regret algorithm? How does the nugget value affect regret behavior?} 
In this paper, we provide an affirmative answer to the first question and guidance to the second question by developing novel theoretical techniques.
Our theoretical results can be used to explain and validate the empirical success of practical EGO.
Our contributions in this paper are two-fold.
\begin{itemize}
    \item First, we establish for the first time a cumulative regret upper bound $\mathcal{O}(\log^{1/2}(T)T^{1/2}\sqrt{\gamma_T})$ for practical EGO, one of the most widely used noise-free BO algorithms, where $\gamma_T$ is the maximum information gain (Definition~\ref{def:infogain}). Thus, we prove that practical EGO is a no-regret algorithm for different kernels, as long as the $\gamma_T$ of a kernel is sublinear.
Specifically, the cumulative regret bounds are $\mathcal{O}(T^{1/2} \log^{(d+2)/2}(T))$ and $\mathcal{O}(T^{\frac{\nu+d}{2\nu+d}}\log^{\frac{2\nu+0.5d}{2\nu+d}} (T)))$ for SE and Mat\'{e}rn  kernels, respectively (see~\eqref{def:sematern} for definitions). 
   \item Second, we study the effect of the nugget $\epsilon$ on practical EGO and its cumulative regret bounds, and thereby providing insight into the choice of $\epsilon$.   
\end{itemize}



This paper is organized as follows. In Section~\ref{se:bo}, we introduce GP, EI, practical EGO, and other necessary background information. In Section~\ref{se:analysis}, we present the regret bound analysis starting with preliminary results in Section~\ref{se:EIproperty}.
 In Section~\ref{se:inst-regret}, the novel instantaneous regret bound is established.  The cumulative regret bound is provided in Section ~\ref{se:regret}.
 Limitation on extending the regret bound analysis to EGO is also presented at the end of Section~\ref{se:analysis}.
 We discuss the effect of $\epsilon$ in Section~\ref{se:nugget-EI}.
 Numerical experiments are used to validate our findings in Section ~\ref{se:examples}. Conclusions
are made in Section ~\ref{se:conclusion}.

\section{Background}\label{se:bo}
In this section, we first provide the basics of GPs and the EI acquisition function. Then, 
other backgrounds relevant to cumulative regret analysis are introduced.
\subsection{Gaussian Process}\label{se:gpbackground}
Consider a zero mean GP with the kernel (\textit{i.e.}, covariance function) $k(\xbm,\xbm'):\Rbb^d\times\Rbb^d\to\Rbb$.. 
Given the $t$th sample  $\xbm_t\in C$, the observed function value is 
$f(\xbm_t)$.  
The $t\times t$ covariance matrix is denoted as $\Kbm_t = [k(\xbm_1,\xbm_1),\dots,k(\xbm_1,\xbm_t);$
$\dots; k(\xbm_t,\xbm_1),\dots,k(\xbm_t,\xbm_t)]$. 
The noise-free observations is $\fbm_{1:t}=[f(\xbm_1),\dots,f(\xbm_t)]^T$.
Without the nugget $\epsilon$, the posterior mean, denoted as $\mu_t^0$, and standard deviation, denoted as $\sigma_t^0$, of the deterministic GP used in EGO is
\begin{equation*} \label{eqn:GP-post-ego}
 \centering
  \begin{aligned}
  &\mu_t^0(\xbm)\ =\ \kbm_t(\xbm) \Kbm_t^{-1} \fbm_{1:t} \\
  &(\sigma^0_t(\xbm))^2\ =\
k(\xbm,\xbm)-\kbm_t(\xbm)^T \Kbm_t^{-1}\kbm_t(\xbm)\ ,
\end{aligned}
\end{equation*}
where $\kbm_t(\xbm)= [k(\xbm_1,\xbm),\dots,k(\xbm_t,\xbm)]^T$. 
Accounting for $\epsilon>0$ and its corresponding scalar matrix $\epsilon\Ibm$, the posterior mean $\mu_t$ and variance  $\sigma^2_t$ used in practical EGO are  
\begin{equation} \label{eqn:GP-post}
 \centering
  \begin{aligned} 
  &\mu_t(\xbm)\ =\ \kbm_t(\xbm) \left(\Kbm_t+ \epsilon \Ibm\right)^{-1} \fbm_{1:t} \\
  &\sigma^2_t(\xbm)\ =\
k(\xbm,\xbm)-\kbm_t(\xbm)^T \left(\Kbm_t+\epsilon \Ibm\right)^{-1}\kbm_t(\xbm)\ ,
\end{aligned}
\end{equation}
We emphasize that the observations $\fbm_{1:t}$ are noise-free in \eqref{eqn:GP-post}.

SE and Matérn kernels are among the most popular kernels for BO and GP. Their definitions are as follows.
\begin{equation} \label{def:sematern}
  \centering
  \begin{aligned}
    &k_{SE}(\xbm,\xbm') = \exp\left(-\frac{r^2}{2l^2}\right), \\
    &k_{Mat\acute{e}rn}(\xbm,\xbm')=\frac{1}{\Gamma
    (\nu)2^{\nu-1}}\left(\frac{\sqrt{2\nu}r}{l}\right)^{\nu} B_{\nu} \left(\frac{\sqrt{2\nu}r}{l}\right).
  \end{aligned}
\end{equation}
where $l>0$ is the length hyper-parameter, $r=\norm{\xbm-\xbm'}_2$, $\nu>0$ is the smoothness parameter of the Matérn kernel, and $B_{\nu}$ is the modified Bessel function of the second kind. 

\subsection{Expected Improvement}
The improvement function of $f$ given $t$ samples is defined as 
\begin{equation} \label{eqn:improvement}
 \centering
  \begin{aligned}
    I_t(\xbm) = \max\{ f^+_{t} -f(\xbm),0  \}, 
  \end{aligned}
\end{equation}
where $f^+_t$ denotes the best current objective value.  
The sample point that generates $f_t^+$ is denoted as $\xbm_t^+$.
The EI acquisition function is defined as the expectation of~\eqref{eqn:improvement} conditioned on $t$ samples, with the expression:
\begin{equation} \label{eqn:EI-1}
 \centering
  \begin{aligned}
       EI_t(\xbm) =   (f_{t}^+-\mu_{t}(\xbm))\Phi(z_{t}(\xbm))+\sigma_{t}(\xbm)\phi(z_{t}(\xbm)),\\ 
  \end{aligned}
\end{equation}
where 
  $z_{t}(\xbm) = \frac{f^+_{t}-\mu_{t}(\xbm)}{\sigma_{t}(\xbm)}$.
%
The functions 
$\phi$ and~$\Phi$ are the PDF and CDF of the standard normal distribution, respectively.
For ease of reference, we refer to $f_t^+-\mu_t(\xbm)$ and $\sigma_t(\xbm)$ as the exploitation and exploration part of $EI_t(\xbm)$, respectively (see also Appendix~\ref{appdx:back} for more discussions).
A commonly used function in the analysis of EI is the function $\tau:\Rbb\to\Rbb$, defined as 
\begin{equation} \label{def:tau}
 \centering
  \begin{aligned}
   \tau(z) = z\Phi(z) + \phi(z). 
  \end{aligned}
\end{equation}
Thus, the $\tau$ form of EI can be written as $EI_t(\xbm) = \sigma_t(\xbm) \tau(z_t(\xbm))$.
The next sample is chosen by maximizing the acquisition function over $C$, \textit{i.e.},  
\begin{equation} \label{eqn:acquisition-1}
 \centering
  \begin{aligned}
      \xbm_{t} = \underset{\substack{\xbm\in C}}{\text{argmax}}  EI_{t-1} (\xbm),
  \end{aligned}
\end{equation}
breaking ties arbitrarily~\citep{frazier2018}.
The practical EGO algorithm is given in Algorithm~\ref{alg:pego}.

\begin{algorithm}[H]
 \caption{Practical EGO}\label{alg:pego}
  \begin{algorithmic}[1]
	  \STATE{Choose $k(\cdot,\cdot)$ and $T_0$ initial samples $\xbm_i, i=0,\dots,T_0$. Observe $f_i$.Train the initial GP. }
  \FOR{$t=T_0+1,T_0+2,\dots$}
	  \STATE{Choose $\xbm_{t}$ using~\eqref{eqn:acquisition-1}.}
	  \STATE{Observe $f(\xbm_{t})$. \;}
	  \STATE{Update the surrogate model using $\xbm_{1:t}$ and $\fbm_{1:t}$.\;}
	  \IF {Evaluation budget is exhausted} 
              \STATE{Exit}
	  \ENDIF 
  \ENDFOR
  \end{algorithmic}
\end{algorithm}
\subsection{Additional Background}\label{se:additionalback}
Denote the optimal function value as $f(\xbm^*)$, where $\xbm^*$ is a global minimum on $C$, \textit{i.e.},  $\xbm^* \in \underset{\substack{\xbm}\in C}{\text{argmin}}  f(\xbm)$. The instantaneous regret  $r_t$ is defined as 
 \begin{equation} \label{eqn:inst-regret}
 \centering
  \begin{aligned}
       r_t = f(\xbm_t)-f(\xbm^*)\geq 0.
  \end{aligned}
\end{equation}
The cumulative regret $R_T$ after $T$ samples is defined as 
\begin{equation} \label{eqn:cumu-regret}
 \centering
  \begin{aligned}
       R_T = \sum_{t=1}^T r_t =\sum_{t=1}^T [f(\xbm_t)-f(\xbm^*)].
  \end{aligned}
\end{equation}

In order to derive the cumulative regret upper bound, we use the well-established maximum information gain results~\citep{srinivas2009gaussian,chowdhury2017kernelized}. Information gain measures the informativeness of a set of sample points in $C$ about $f$. 
The maximum information gain $\gamma_t$ is defined in Definition~\ref{def:infogain} in Appendix~\ref{appdx:back}.
It is often used to bound the summation of posterior standard deviation $\sigma_{t-1}(\xbm_t)$ and is dependent on the choice of the kernel~\cite{srinivas2009gaussian}.
We note that the upper bound on the sum of $\sigma_{t-1}(\xbm_t)$ is dependent on both $\gamma_t$ and $\epsilon$. 
Moreover, the bound on $\gamma_t$ itself is also dependent on $\epsilon$, the kernel, $C$, and $d$. 
The latest bounds on $\gamma_t$ in literature for common kernels such as the SE kernel and Mat\'{e}rn kernel  can be found in~\cite{vakili2021information,iwazaki2025improved} and  Lemma~\ref{lem:gammarate}.


\section{Regret Bound}\label{se:analysis}
In this section, we present our regret bounds of practical EGO. We start with preliminary results required for the analysis in Section~\ref{se:EIproperty}. Then, the new instantaneous regret bound is given in Section~\ref{se:inst-regret}. Finally, the cumulative regret bound is established in Section~\ref{se:regret}.
\subsection{Assumptions and Preliminary Results}\label{se:EIproperty}
Throughout this paper, we consider the frequentist setting. That is, $f$ lies in the RKHS of $k(\cdot,\cdot)$, a common assumption in literature~\citep{srinivas2009gaussian}, whose definition is in Section~\ref{appdx:back} in the appendix. 
The formal assumption is given below.
\begin{assumption}\label{assp:rkhs}
	The function $f$ lies in the RKHS, denoted as $\mathcal{H}_k(C)$, associated with the bounded kernel $k(\xbm,\xbm')$ with the norm $\norm{\cdot}_{H_k}$. 
         The kernel satisfies $k(\xbm,\xbm')\leq 1$,  $\forall \xbm,\xbm'\in C$, and $k(\xbm,\xbm)= 1$.
	The RKHS norms of the kernels are bounded above by constant $B>0$, \textit{i.e.}, $\norm{f}_{H_k} \leq B$.  The set $C$ is compact.
\end{assumption}
To help analyze $r_t$, we present Lemmas~\ref{lem:tau} to~\ref{lemma:gp-sigma-bound} on the properties of GP and EI.
We briefly summarize some of them here, and leave the theory statements and proofs in Appendix~\ref{se:prep-proof}. 
The monotonicity of function $\tau$ \eqref{def:tau} and its derivative are given in Lemma~\ref{lem:tau}~\citep{jones1998efficient}.
Lemma~\ref{lem:tauvsPhi} states the relationship between $\Phi(z)$ and $\tau(z)$ when $z<0$.
In Lemma~\ref{lem:EI}, we establish simple but useful bounds of for $\frac{EI_{t-1}(\xbm)}{\sigma_{t-1}(\xbm)}$.
Since $EI_{t-1}(\xbm)=\sigma_{t-1}(\xbm)\tau(z_{t-1}(\xbm))$, the above three lemmas can be used to bound $EI_{t-1}(\xbm)$.
In Lemma~\ref{lem:EI-ms}, $EI_{t-1}(\xbm)$ is shown to be monotonically increasing with respect to both its exploitation $f_{t-1}^+-\mu_{t-1}(\xbm)$ and exploration $\sigma_{t-1}(\xbm)$.
Lemma~\ref{lem:tauvsPhi},~\ref{lem:EI} and~\ref{lem:EI-ms}  are used to quantify the exploration and exploitation trade-off properties in Section~\ref{se:inst-regret}. 

A lower bound on $f_{t-1}^+-\mu_{t-1}(\xbm_t)$ when $EI_{t-1}(\xbm)$ is bounded below is given in Lemma~\ref{lem:mu-bounded-EI}~\citep{nguyen17a}. 
The global lower bound for $\sigma_{t-1}(\xbm)$ with nugget is given in Lemma~\ref{lemma:gp-sigma-bound}.
These two lemmas are used to establish the lower bound for $EI_{t-1}(\xbm_t)$ and subsequently a lower bound for $f_{t-1}^+-\mu_{t-1}(\xbm_t)$, which appears in an intermediate upper bound for $r_t$.
  
Next, we establish the bound on $|I_{t-1}(\xbm)-EI_{t-1}(\xbm)|$, an important step leading to the bound on $r_t$.
Under the noise-free frequentist setting, the upper bound on  $|f(\xbm)-\mu_{t-1}(\xbm)|$ has been established in literature(see \textit{e.g.},~\cite{srinivas2009gaussian}). We note that while practical EGO uses $\epsilon$ in its posterior calculations, the confidence interval of $|f(\xbm)-\mu_{t-1}(\xbm)|\leq B\sigma_{t-1}(\xbm)$ continue to hold~\cite{chowdhury2017kernelized}. The effect of $\epsilon$ on the upper bound is reflected in the increased $\sigma_{t-1}(\xbm)$. 
Specifically, $|f(\xbm)-\mu_{t-1}(\xbm)| \leq B \sigma_{t-1}(\xbm)$ holds at given $\xbm\in C$ and $t\in\Nbb$, as stated in Lemma~\ref{lem:fmu}.
Then, using this bound, 
we can establish the bounds on $I_{t-1}(\xbm)$ and $EI_{t-1}(\xbm)$ in Lemma~\ref{lem:IEI-bound-ratio}  through Lemma~\ref{lem:fandnu}.

\subsection{Instantaneous Regret Bound}\label{se:inst-regret}
In this section, we derive the instantaneous regret upper bounds of $r_t$ in terms of the posterior standard deviations $\sigma_{t-1}(\xbm_t)$ and additional exploitation terms, where the former's sum can be bounded with maximum information gain $\gamma_t$. 


\begin{lemma}\label{lem:ego-instregret-1}
  The practical EGO generates the instantaneous regret bound 
  \begin{equation} \label{eqn:ego-instregret-1}
  \centering
  \begin{aligned}
          r_t \leq& c_{B1} \max\{f_{t-1}^+-f(\xbm_t),0\}\\
          &+(c_{B\epsilon}(\epsilon, t)+B+c_B (B+\phi(0)))\sigma_{t-1}(\xbm_{t}),
  \end{aligned}
  \end{equation}
  where $c_{B\epsilon}(\epsilon,t)= \log^{1/2}\left(\frac{t+\epsilon}{2\pi  \epsilon \tau^2(-B)}\right)$, $c_B=\frac{\tau(B)}{\tau(-B)}$ and $c_{B1}=\max\left\{\frac{\tau(B)}{\tau(-B)}-1,0\right\}$.
\end{lemma}

\textbf{Proof Sketch for Lemma~\ref{lem:ego-instregret-1}.}
We consider two different cases: $f^+_{t-1}-f(\xbm_t)\leq 0$ and $f_{t-1}^+-f(\xbm_t)> 0$.
For the first case, where $f_{t-1}^+-f(\xbm_t) \leq 0$,
we use the bound on $|f(\xbm)-\mu_{t-1}(\xbm)|$ (Lemma~\ref{lem:fmu}) and Lemma~\ref{lem:IEI-bound-ratio} to obtain an upper bound of $r_t$: $\mu_{t-1}(\xbm_t)-f_{t-1}^++c_B EI_{t-1}(\xbm_t)+B\sigma_{t-1}(\xbm_t)$, where $c_B=\frac{\tau(B)}{\tau(-B)}$. 
To derive an upper bound for $\mu_{t-1}(\xbm_t)-f_{t-1}^+$, we establish a positive lower bound $\mathcal{O}\left(\frac{1}{\sqrt{t}}\right)$ of $EI_{t-1}(\xbm_t)$ at $\forall t\in\Nbb$.  We consider $EI_{t-1}(\xbm^*)$ and show that $EI_{t-1}(\xbm^*)\geq \sigma_{t-1}(\xbm^*)\tau(-B)$ using the global lower bound on $\sigma_{t-1}(\xbm)$ mentioned in Section~\ref{se:EIproperty} (Lemma~\ref{lemma:gp-sigma-bound}).
Then, the upper bound for $\mu_{t-1}(\xbm_t)-f_{t-1}^+$ can be derived by the properties of EI (Lemma~\ref{lem:mu-bounded-EI}). 
The upper bound on $r_t$ becomes $(c_{B\epsilon}(\epsilon,t) + B+c_B(\phi(0)+B))\sigma_{t-1}(\xbm_t) $, where $c_{B\epsilon}(\epsilon,t)= \log^{1/2}\left(\frac{t+\epsilon}{2\pi \tau^2(-B) \epsilon}\right)$.

For the second case where $f_{t-1}^+-f(\xbm_t) \geq 0$, from Lemma~\ref{lem:IEI-bound-ratio}, 
we have the upper bound for $r_t$: $f(\xbm_t)-f_{t-1}^++c_BEI_{t-1}(\xbm_t)$. 
We can further bound $EI_{t-1}(\xbm_t)$ via the bounds on $|f(\xbm)-\mu_{t-1}(\xbm)|$ (Lemma~\ref{lem:fmu}) and the properties of EI (Lemma~\ref{lem:EI}). Thus, the upper bound for $r_t$ becomes $(c_B-1) (f_{t-1}^+-f(\xbm_t))  +c_B(B+\phi(0))\sigma_{t-1}(\xbm_{t})$. 
Combining the bounds in both cases leads to~\eqref{eqn:ego-instregret-1} in Lemma~\ref{lem:ego-instregret-1}.


\begin{remark}[Use of $\epsilon$ in Lemma~\ref{lem:ego-instregret-1}]\label{remark:nugget} In addition to improved numerical stability and statistical properties mentioned in Section~\ref{se:introduction}, $\epsilon$ plays an important role in the analysis of $r_t$. Specifically, the upper bound on $\mu_{t-1}(\xbm_t)-f_{t-1}^+$ requires a positive lower bound on $EI_{t-1}(\xbm_t)$. The use of $\epsilon$ provides a positive global lower bound for the posterior standard deviation $\sigma_{t-1}(\xbm)$, which leads to the positive lower bound on $EI_{t-1}(\xbm_t)$. 
Without $\epsilon$, at previous sample points $\xbm_i, i=1,\dots,t-1$, we have $\sigma_{t-1}(\xbm_i) = 0$ and $EI_{t-1}(\xbm_i) =0$. Since it is possible that $\xbm_i=\xbm^*$ for some $i$, we can no longer guarantee a positive lower bound for $EI_{t-1}(\xbm^*)$, a critical step towards the lower bound on $EI_{t-1}(\xbm_t)$.
Thus, we can no longer obtain a desirable upper bound for $\mu_{t-1}(\xbm_t)-f_{t-1}^+$. Intuitively, $\epsilon>0$ means that there is still uncertainty recognized by the GP model, albeit decreasing, at previous sample points, making it more likely that the next sample is chosen close to existing samples, when some of them are already close to $\xbm^*$.
\end{remark}

\begin{remark}[Exploitation term in instantaneous regret bound]\label{remark:instreg-exploitation}
   The instantaneous regret upper bound in Lemma~\ref{lem:ego-instregret-1} contains the exploitation term $\max\{f_{t-1}^+-f(\xbm_t),0\}$, which is a novelty for instantaneous regret bound, as far as we know. For instance, the instantaneous regret bound for UCB only has the exploration terms involving $\sigma_{t-1}(\xbm_t)$. We elaborate our techniques to bound the sum of $\max\{f_{t-1}^+-f(\xbm_t),0\}$ in Section~\ref{se:regret}.
\end{remark}
\subsection{Cumulative Regret Bound}\label{se:regret}
In this section, we present the cumulative regret bound of practical EGO. 
The following lemma establishes the bound based on Lemma~\ref{lem:ego-instregret-1}.

\begin{lemma}\label{lemma:boi-inst-regret}
   The cumulative regret bound of practical EGO satisfies  
   \begin{equation} \label{eqn:boi-inst-regret-1}
  \centering
  \begin{aligned}
     R_T \leq 2 c_{B1} B + (& c_{B\epsilon}(\epsilon,T)+B\\
      &+c_B (B+\phi(0)))\sqrt{C_{\gamma}(\epsilon)T\gamma_T},   \\
    \end{aligned}
  \end{equation}
  where $c_{B\epsilon}(\epsilon,t)= \log^{\frac{1}{2}}(\frac{t+\epsilon}{2\pi  \epsilon \tau^2(-B)})$, $C_{\gamma}(\epsilon)=\frac{2}{\log(1+1/\epsilon)}$, $c_B=\frac{\tau(B)}{\tau(-B)}$, and $c_{B1}=\max\left\{\frac{\tau(B)}{\tau(-B)}-1,0\right\}$.
\end{lemma}

\paragraph{Proof Sketch for Lemma~\ref{lemma:boi-inst-regret}.}
 From Lemma~\ref{lem:ego-instregret-1} and the definition of $R_T$, we need to bound the sum of the exploitation term $\sum_{t=1}^T \max\{f_{t-1}^+-f(\xbm_t),0\}$ and $\sum_{t=1}^T c_{B\epsilon}(\epsilon,t)\sigma_{t-1}(\xbm_{t})$, given that the sum of the remaining terms are obvious. 
 To bound the first sum, we construct the subsequence $\{\xbm_{t_i}\}$ of $\{\xbm_t\}$ for all $\xbm_t$ that satisfies $f_{t-1}^+-f(\xbm_t)>0$. Using $f_{t_i}^+\leq f(\xbm_{t_i})$ and the fact that $t_{i-1} \leq t_i-1$, we can write $f_{t_i-1}^+-f(\xbm_{t_i})+f_{t_{i+1}-1}^+-f(\xbm_{t_{i+1}})\leq f(\xbm_{t_{i-1}})-f(\xbm_{t_{i+1}})\leq 2B$. 
Using this technique and summing up all $t_i$ lead to the bound on  $\sum_{t=1}^T \max\{f_{t-1}^+-f(\xbm_t),0\}$. 

For the latter term, using the maximum information gain $\gamma_t$ from Lemma~\ref{lem:variancebound} suffices as $\sum_{t=1}^T c_{B\epsilon}(\epsilon,t)\sigma_{t-1}(\xbm_{t})\leq  c_{B\epsilon}(\epsilon,T) \sum_{t=1}^T\sigma_{t-1}(\xbm_{t})$. 

The rate of the cumulative regret upper bound of practical EGO is stated in the following theorem.
\begin{theorem}\label{thm:boi-cumulative-regret}
    The practical EGO algorithm leads to the cumulative regret upper bound  
    \begin{equation*} \label{eqn:ts-regret-1}
  \centering
  \begin{aligned}
       R_T = \mathcal{O}(\log^{1/2}(T)T^{1/2}\sqrt{\gamma_T}).
    \end{aligned}
  \end{equation*}
   For SE kernel,  $R_T=\mathcal{O}(T^{1/2} \log^{(d+2)/2}(T))$.
   For Matérn kernels ($\nu>\frac{1}{2}$), $R_T=\mathcal{O}(T^{\frac{\nu+d}{2\nu+d}}\log^{\frac{2\nu+0.5d}{2\nu+d}} (T))$.\\
\end{theorem}
From Theorem~\ref{thm:boi-cumulative-regret}, if $\gamma_T$ of the chosen kernel is sublinear, practical EGO is no-regret.
\begin{remark}
    Our analysis framework does not directly apply to the cumulative regret of EGO (without $\epsilon$). As mentioned in Remark~\ref{remark:nugget}, EGO has $\sigma_{t-1}(\xbm_i)=0, i=1,\dots,t-1$, and thus $EI_{t-1}(\xbm_i)=0$ at previous sample points. Therefore,
    we cannot obtain a lower bound of $EI_{t-1}(\xbm_t)$ via $EI_{t-1}(\xbm^*)$. Consequently, the current upper bound on $\mu_{t-1}(\xbm_t)-f_{t-1}^+$ cannot be obtained for EGO. 
    Determining whether EGO is a no-regret algorithm is a topic for future research and can be considered a limitation of this paper.
\end{remark}
\section{Effect of the Nugget}\label{se:nugget-EI}
In this section, we discuss the impact of the nugget $\epsilon$ on practical EGO. 
The effect of $\epsilon$ on the properties of GP such as the likelihood have been studied previously~\citep{andrianakis2012effect}. Our focus is thus on how $\epsilon$ can change EI and the cumulative regret bounds of practical EGO.

 First, we briefly demonstrate the effect $\epsilon$ can have on the maximum of the EI function, \textit{i.e.}, $EI_{t-1}(\xbm_t)$, via a 2-dimensional example. We use the Branin function (see  example 4 in Section~\ref{se:examples}) and the GP model from scikit-learn~\citep{scikit-learn}. We train the GP with $\epsilon=10^{-2}$, $\epsilon=10^{-6}$, $\epsilon=10^{-10}$, and without nuggets using the same $50$ samples ($25$ initial samples and $25$ iterations of practical EGO where we set $\epsilon=10^{-6}$), and plot the contours of EI in Figure~\ref{fig:ei-nugget}. We note that the maximum level of the EI value colorbar corresponds to the maximum of EI$_{50}$ in each contour plot.
 
\begin{figure*}
  \centering
  \includegraphics[width=0.9\textwidth]{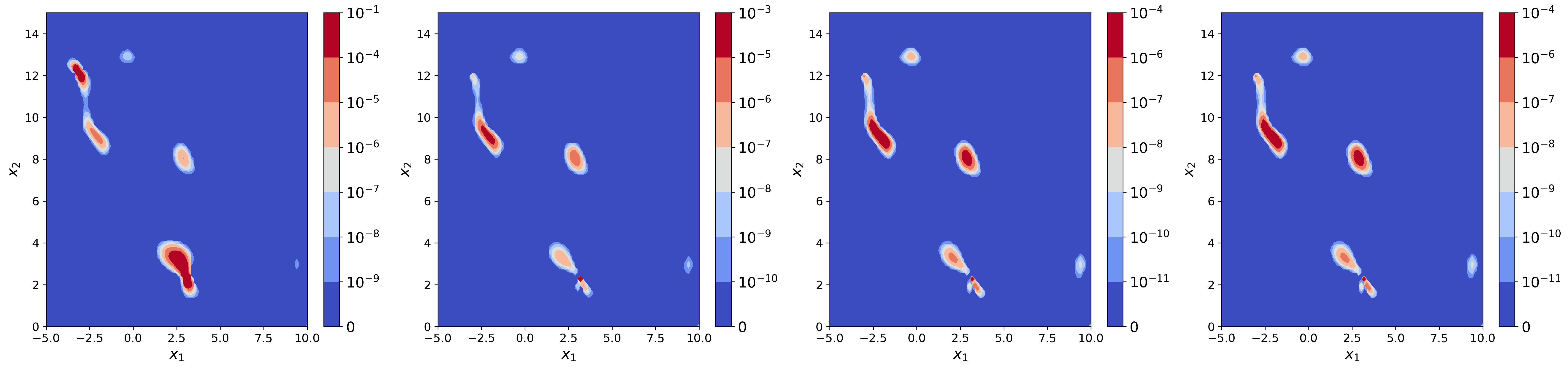}
        \caption{Illustrative example of EI contour of the Branin function with 50 samples. From left to right: contour plots for $\epsilon=10^{-2}$, $\epsilon=10^{-6}$, $\epsilon=10^{-10}$, and no nugget. The maximum EI$_{50}$ value from left to right: 
        $2.53\times 10^{-2}$, $1.90\times 10^{-4}$, $4.72\times 10^{-5}$, and $4.72\times 10^{-5}$.}
\label{fig:ei-nugget}
\end{figure*}

\begin{remark} 
We remark that while given $50$ samples, GP with no nugget does not have numerical issues, we encounter Cholesky decomposition failure when inverting the covariance matrix using $25$ random initial samples and around $75$ samples generated from practical EGO runs. 
\end{remark}
It is clear from Figure~\ref{fig:ei-nugget} that $\epsilon$ impacts $EI_{t-1}(\xbm_t)$, of which a positive lower bound  is needed for the bound on $r_t$ (see proof of Lemma~\ref{lem:ego-instregret-1}). In this example, the impact is relatively small when $\epsilon
$ is small, \textit{e.g.}, $10^{-10}$. Consistent with our analysis in  Section~\ref{se:analysis}, $\epsilon$ ensures a lower bound on $\sigma_{t-1}(\xbm)$ and thus a lower bound of $EI_{t-1}(\xbm_t)$. This is reflected in the increased value of $EI_{t-1}(\xbm_t)$ as $\epsilon$ increases. We emphasize that the effect of $\epsilon$ is  dependent on multiple factors such as the sample set, the function, the kernel, etc, and therefore, Figure~\ref{fig:ei-nugget} is one illustrative example. For instance, if one uses all $50$ samples generated by random sampling in the example above, the effect of $\epsilon$ would be much less prominent since the samples are more evenly spaced out, as shown in Figure~\ref{fig:ei-nugget-random} in appendix.

Next, we discuss the effect of the  nugget on the regret bound from the theoretical perspective. 
From Lemma~\ref{eqn:boi-inst-regret-1}, $\epsilon$ appears in the cumulative regret upper bound 
    \begin{equation} \label{eqn:upper-bound}
  \centering
  \begin{aligned}
2 c_{B1} B + (c_{B\epsilon}(\epsilon,T)+B+c_B (B+\phi(0))) \sqrt{C_{\gamma}(\epsilon) T\gamma_T(\epsilon)}
\end{aligned}
  \end{equation}
through $c_{B\epsilon}(\epsilon,T)$, $C_{\gamma}(\epsilon)$, as well as $\gamma_T(\epsilon)$.  

To make matters complicated, $c_{B\epsilon}(\epsilon,T)$ decreases, while $C_{\gamma}(\epsilon)$ increases, as $\epsilon$ increases. 
Further, the dependence of $\gamma_T$  on $\epsilon$ is kernel specific. Thus, the impact of $\epsilon$ on~\eqref{eqn:upper-bound} is  complex and dependent on the kernel. Here, we provide an answer to this challenging question for SE and Matérn kernels via the bounds on $\gamma_T$ in~\cite{iwazaki2025improved}, as stated in Lemma~\ref{thm:MIG}. 
From Lemma~\ref{thm:MIG} for SE and Matérn kernels, the upper bound on $\gamma_T(\epsilon)$ decreases as $\epsilon$ increases.

For simplicity, we consider constant length scale $l$ in both kernels. Further, since the nugget is often small in nature, we focus on the case where $T/\epsilon$ is large.
The effect of the nugget $\epsilon$ on~\eqref{eqn:upper-bound}for SE kernel is presented next. 
  \begin{theorem}\label{cor:ego-nugget-se}
   Under the conditions of Lemma~\ref{lem:mig-se} and $T/\epsilon\gg 1$,
for SE kernel at given $T$,\\
(1) if 
\begin{equation} \label{eqn:ego-nugget-se-1}
  \centering
  \begin{aligned}
     (d+1) \log(1+1/\epsilon) > \log(1+T/\epsilon),
 \end{aligned}
  \end{equation}
 and $\log(1+T/\epsilon)\gg \max\{C_{dl}^2,C_{dl}^3,C_R^2,C_R^3,C_R^3,d\}$, then~\eqref{eqn:upper-bound} decreases as $\epsilon$ increases. The constants are defined in Lemma~\ref{thm:MIG} and Lemma~\ref{lem:mig-se}. \\
(2) if 
\begin{equation} \label{eqn:ego-nugget-se-2}
  \centering
  \begin{aligned}
(d+2) \log(1+1/\epsilon) < (1+\epsilon/T)/(1+\epsilon) \log(1+T/\epsilon),
\end{aligned}
  \end{equation}
  then~\eqref{eqn:upper-bound} increases  as $\epsilon$ increases.
\end{theorem}
  
Next, we consider the Matérn kernels $(\nu>\frac{1}{2})$. 
\begin{theorem}\label{cor:ego-regret-nugget-matern}
   Under the conditions of Lemma~\ref{lem:mig-matern-shogo} where $T/\epsilon\gg 1$,
for Matérn kernel ($\nu>\frac{1}{2}$),

(1) if 
\begin{equation} \label{eqn:ego-nugget-matern-1}
  \centering
  \begin{aligned}
 &\log(1+1/\epsilon)d/(2\nu+d)>(1+1/C_{d\nu l}^2), \ \text{and}\\
  &C_{\nu}^1C_{d\nu l}^1C_{\nu}^3\log(1+1/\epsilon)\log(1+2T/\epsilon)>C,
\end{aligned}
  \end{equation}
then~\eqref{eqn:upper-bound} decreases with increasing $\epsilon$. The constants are defined in Lemma~\ref{thm:MIG} and Lemma~\ref{lem:mig-matern-shogo}.\\  
(2) 
if 
\begin{equation} \label{eqn:ego-nugget-matern-2}
  \centering
  \begin{aligned}
\log(1+1/\epsilon)\left[\frac{d}{2\nu+d}+C_{\nu}^1C_{\nu}^3+\left(\frac{4\nu+d}{2\nu+d} \right.\right.\\
      \left. \left.  +C_{\nu}^1\right)\frac{1}{\log(T/\epsilon)}\right] < 1/(1+\epsilon), 
\end{aligned}
  \end{equation}
  then~\eqref{eqn:upper-bound} increases with increasing $\epsilon$. 
\end{theorem}
\begin{remark}
    The constants in Theorem~\ref{cor:ego-nugget-se} and~\ref{cor:ego-regret-nugget-matern} are dependent on $d$, $C$, and the fixed hyper-parameter $l$. Readers are referred to Lemma~\ref{thm:MIG},~\ref{lem:mig-se}, and~\ref{lem:mig-matern-shogo} for their definitions.
    When $T/\epsilon$ is sufficiently large, the conditions involving the constants are satisfied. 
\end{remark}

\begin{remark}
    We note that case 2 in both theorems might not be satisfied when $T$ is small ( $\epsilon$ also small to maintain a large $T/\epsilon$), as well as when $d$ is large. 
    Further, for any given $T$, for SE kernel, it is possible that neither~\eqref{eqn:ego-nugget-se-1} nor~\eqref{eqn:ego-nugget-se-2} is satisfied. In such cases, the constants play an important role in how $\epsilon$ affect~\eqref{eqn:upper-bound}. Similar conclusions can be drawn for Matérn kernels.
\end{remark}

Theorem~\ref{cor:ego-nugget-se} and~\ref{cor:ego-regret-nugget-matern} show that to obtain a tighter cumulative regret bound, $\epsilon$ should stay within a reasonable range.  
When $\epsilon$ is small, conditions $(1)$ for both kernels are more likely to be satisfied and~\eqref{eqn:upper-bound} increases as $\epsilon$ decreases. When $\epsilon$ is large, conditions $(2)$ are more likely to be satisfied and~\eqref{eqn:upper-bound} increases as $\epsilon$ increases. 

Intuitively, if $\epsilon$ is too large, the posterior variance is inflated too much and EI could emphasize too much on exploration. On the other hand, if $\epsilon$ is too small, practical EGO behaves closer to EGO, which might not be no-regret. In addition, an $\epsilon$ too small risks not resolving numerical stability issues.
We emphasize that our analysis is based on state-of-the-art cumulative regret bound~\eqref{eqn:upper-bound}, and not the cumulative regret itself. 

To better illustrate our results, we choose an example set of constants and plot the cumulative regret bound with $\epsilon$. 
Let $C_{dl}^1=C_{dl}^2=C_{dl}^3=1$, $d=2$, $B=1$ for the the SE kernel. For the Matérn kernel, let $\nu=2.5$, $d=3$, $C_{\nu}=1$, $C_{d\nu l 1}=1$, $C_{d\nu l 2}=1$, and $C=1$. 
The rest of the constants can be deduced from these chosen ones.  
We plot~\eqref{eqn:upper-bound} with $\epsilon$ for SE kernel in Figure~\ref{fig:regret-se-nugget} and the one for Matérn kernel in Figure~\ref{fig:regret-matern-nugget} in the appendix. We mark when the conditions for the two cases are met in Theorem~\ref{cor:ego-nugget-se} and~\ref{cor:ego-regret-nugget-matern}. 

\begin{figure}
	\centering
	\includegraphics[width=0.9\linewidth]{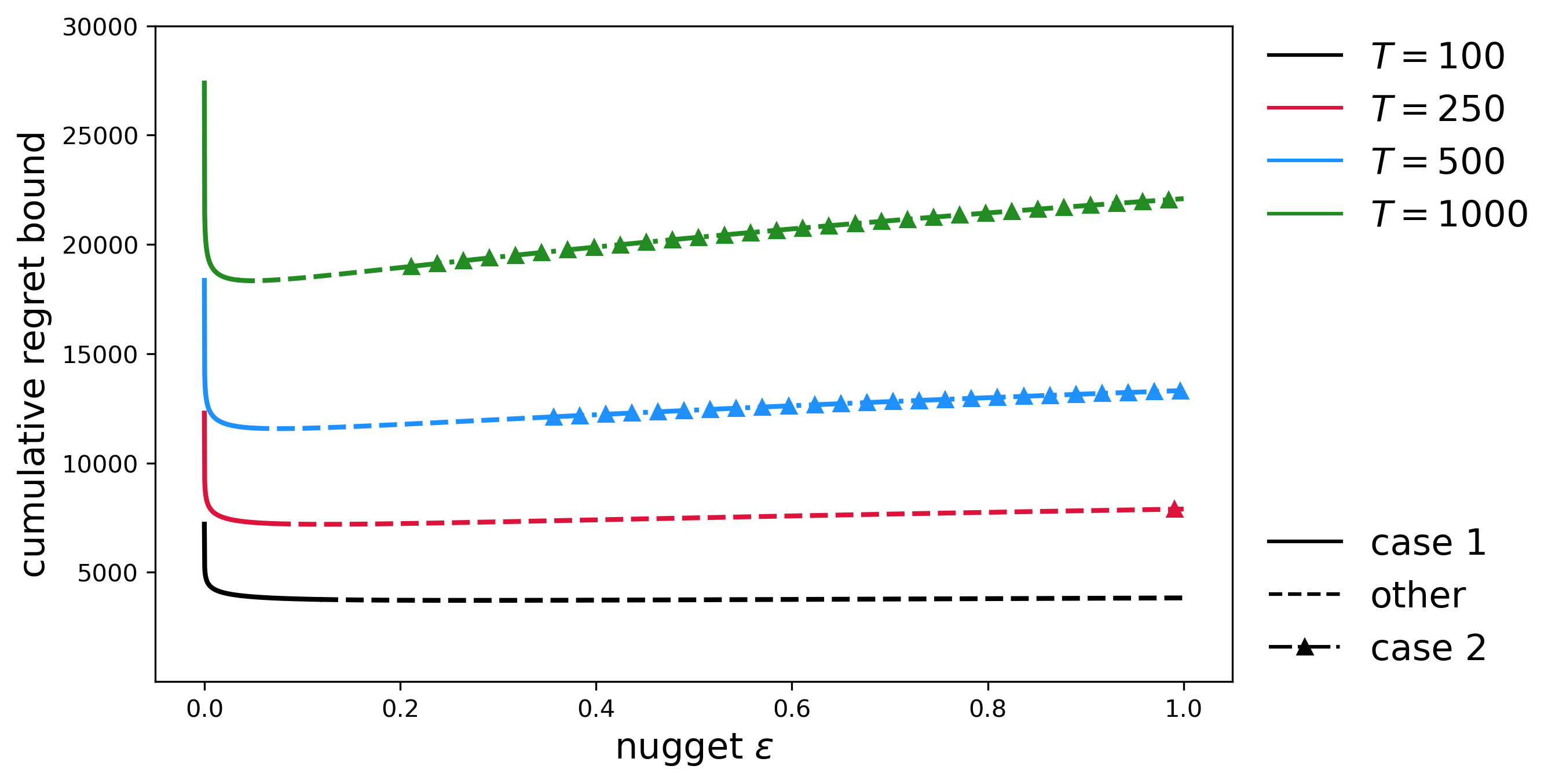}
        \caption{Cumulative regret upper bound with nugget $\epsilon$ at different $T$ and selected constants for SE kernel. The case ``other'' means neither the conditions for case $1$ nor those for case $2$ are satisfied.}
    \label{fig:regret-se-nugget}
\end{figure}
The plots clearly demonstrate our theoretical conclusions on the effect of $\epsilon$. 
We note that the inequalities~\eqref{eqn:ego-nugget-se-1},~\eqref{eqn:ego-nugget-se-2},~\eqref{eqn:ego-nugget-matern-1}, and~\eqref{eqn:ego-nugget-matern-2} are much relaxed to allow for simpler forms. It is clear that when $T$ is small, the second case conditions for both kernels might not be satisfied. 
This does not mean that the cumulative regret bound is not increasing with $\epsilon$. However, when $T$ is small the effect of $\epsilon$ becomes more complicated to summarize and highly dependent on constants.

\section{Numerical Experiments}\label{se:examples}
In this section, we perform numerical experiments of practical EGO with varying nugget values on widely used test problems to demonstrate the empirical validity of our theories. 
We consider two groups of functions. First, we test functions sampled from GPs with known SE and Matérn kernels ($\nu=2.5$) using fixed hyperparameters. These problems are designed to minimize the effect of misspecification of kernels or hyperparameter optimization of GP. Second, we consider commonly used synthetic functions.

For the first group of functions, we use 20 and 40 initial design points for 2D and 4D problems, respectively, followed by 200 additional observations acquired iteratively via practical EGO. 
We use three $\epsilon$ values for each problem, $10^{-10},10^{-6},$ and $10^{-4}$.
The GP hyperparameters were kept fixed to the original values used for sampling. The results are shown in Table~\ref{gp-table}. 

For each problem, the average cumulative regret, \textit{i.e.}, $R_T/T$, declines as optimization progresses. 
For Matérn kernels and SE kernel in 4D, the observed regret behavior with respect to $\epsilon$ align well with our theoretical regret bound predictions, namely that regret bounds do not change monotonically with $\epsilon$ and  there might exist a range where $\epsilon$ should be chosen. 
The SE kernels in 2D however, show a smaller regret for smaller $\epsilon$. We note that this is not contradictory to our conclusion, as it is based on the upper bound of the regret. 
\begin{table}[t]
  \caption{The average accumulative regret
 for different iteration over 20 macro-replications.}
  \label{gp-table}
  \begin{center}
    \begin{small}
      \begin{sc}
        \resizebox{\columnwidth}{!}{%
        \begin{tabular}{lcccccr}
          \toprule
          d & kernel  & $\epsilon$   & t=1 & t=50 & t=100    & t=200  \\
          \midrule
          2&	SE&	$10^{-10}$&	0.596&	0.091&	0.058&	0.040\\
          2&	SE&	$10^{-6}$&	0.596&	0.141&	0.101&	0.076\\
          2&	SE&	$10^{-4}$&	0.596&	0.153&	0.120&	0.097\\
          2&	Matérn&	$10^{-10}$&	0.393&	0.061&	0.031&	0.016\\
          2&	Matérn&	$10^{-6}$&	0.393&	0.055&	0.028&	0.015\\
          2&	Matérn&	$10^{-4}$&	0.393&	0.054&	0.028&	0.016\\
          4&	SE&	$10^{-10}$&	1.422&	0.775&	0.626&	0.440\\
          4&	SE&	$10^{-6}$&	1.422&	0.788&	0.640&	0.440\\
          4&	SE&	$10^{-4}$&	1.422&	0.799&	0.693&	0.528\\
          4&	Matérn&	$10^{-10}$&	0.857&	0.458&	0.296&	0.183\\
          4&	Matérn&	$10^{-6}$&	0.857&	0.493&	0.291&	0.184\\
          4&	Matérn&	$10^{-4}$&	0.857&	0.474&	0.292&	0.184\\
          \bottomrule
        \end{tabular}
        }
      \end{sc}
    \end{small}
  \end{center}
  \vskip -0.1in
\end{table}

For synthetic problems, we choose five examples and three nugget values $\epsilon=10^{-2}$, $\epsilon=10^{-4}$, and $\epsilon=10^{-6}$ for each problem. 
From example 1 to 5, the objective functions are the Rosenbrock function, the six-hump camel function, the Hartmann6 function, the Branin function, and the Michalewicz function~\citep{molga2005test,picheny2013benchmark}. 
The mathematical expressions of our test problems are given in Section~\ref{appx:example} in the appendix. 
The results for the five examples are plotted in Figure~\ref{fig:numerical}.

The synthetic examples are implemented using GP models in scikit-learn.
Each example is run 100 times to account for stochasticity. 
For the two-dimensional examples 1, 2, 4, 5, five initial Latin hypercube samples are used. For the six-dimensional problem, example 3, we choose $50$ initial Latin hypercube samples due to the increase in dimension. 
For example 1, 2, 3, SE kernel is used, while Matérn kernel with $\nu=2.5$ is used for example 4 and 5, for a demonstration of both kernels mentioned in our theories.
The median of the average cumulative regret  among the 100 runs is reported against the number of iterations. The $25$th and $75$th percentile results are shown in Section~\ref{appx:example} in the appendix.
The computational budget is $200$ optimization  iterations for examples 1,2, and 4 and $100$ for examples 3, due to the increased dimension and computational cost. For example 5, we also terminate at $100$ optimization iterations as $R_T/T$ is sufficiently small.  

\begin{figure}
  \centering
    \begin{subfigure}{\linewidth}
    \centering
    \includegraphics[width=0.75\linewidth]{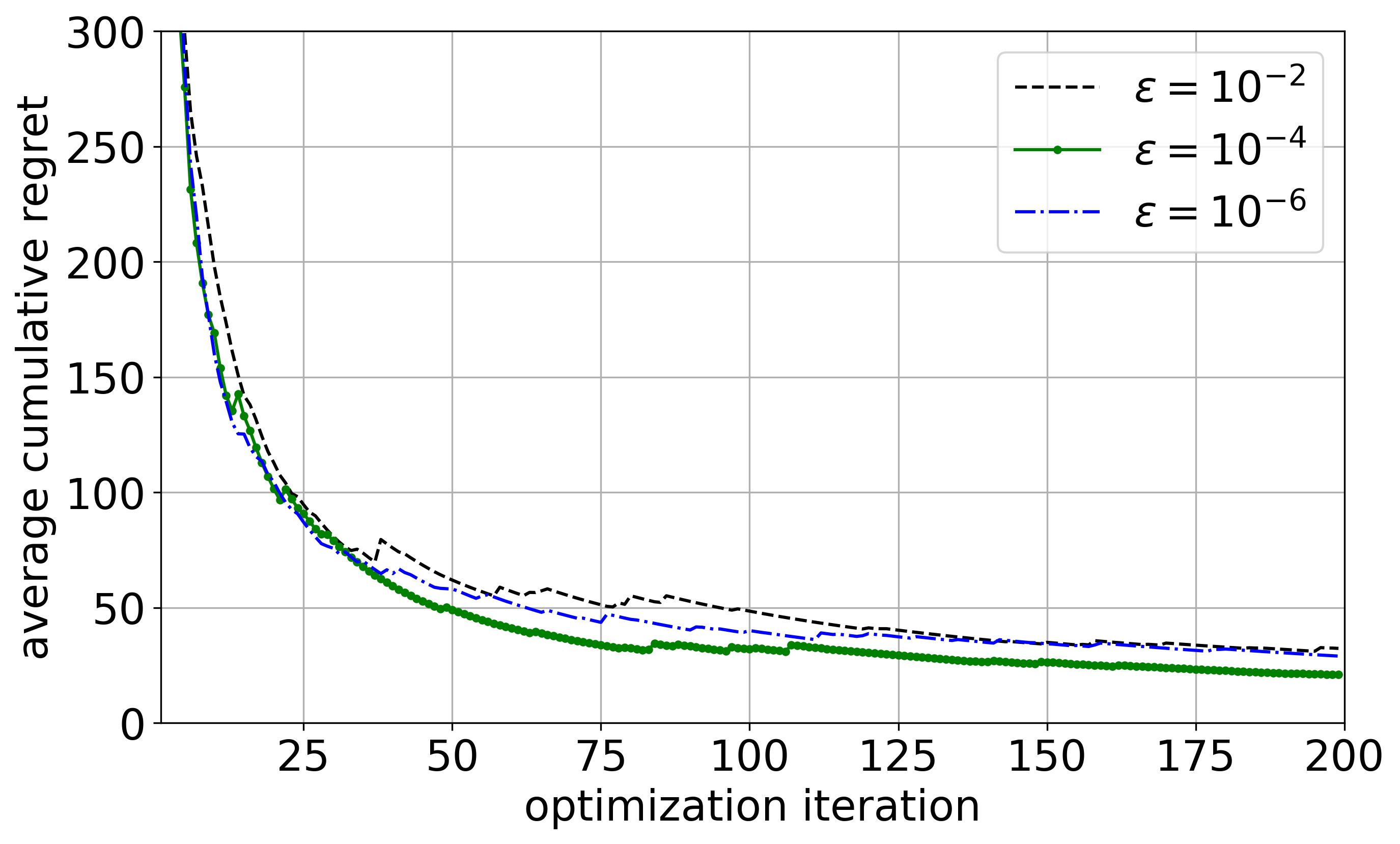}
    \label{fig:ex2}
    \end{subfigure}
   \begin{subfigure}{\linewidth}
    \centering
    \includegraphics[width=0.75\linewidth]{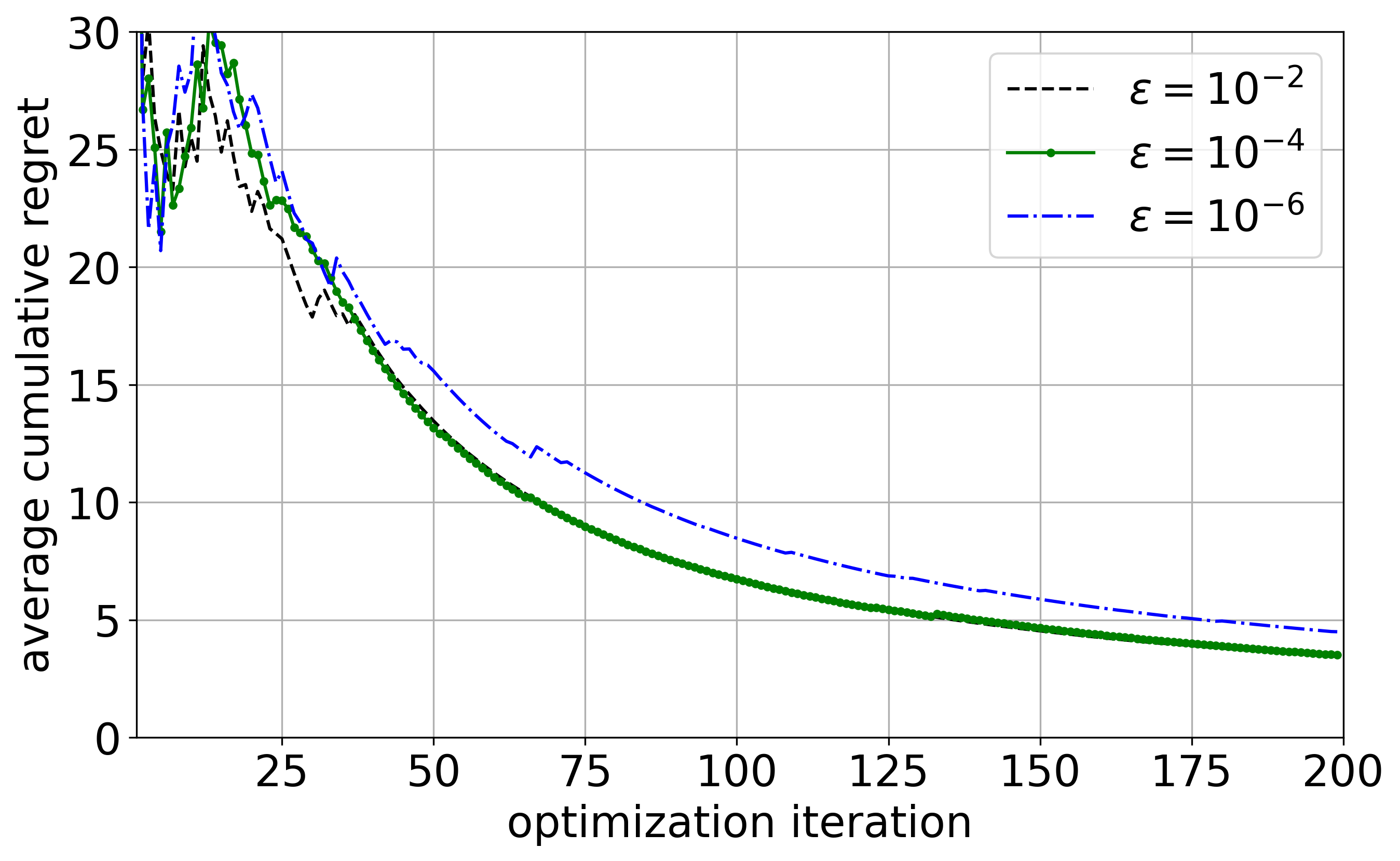}
    \label{fig:ex3}
    \end{subfigure}
    \begin{subfigure}{\linewidth}
    \centering
    \includegraphics[width=0.75\linewidth]{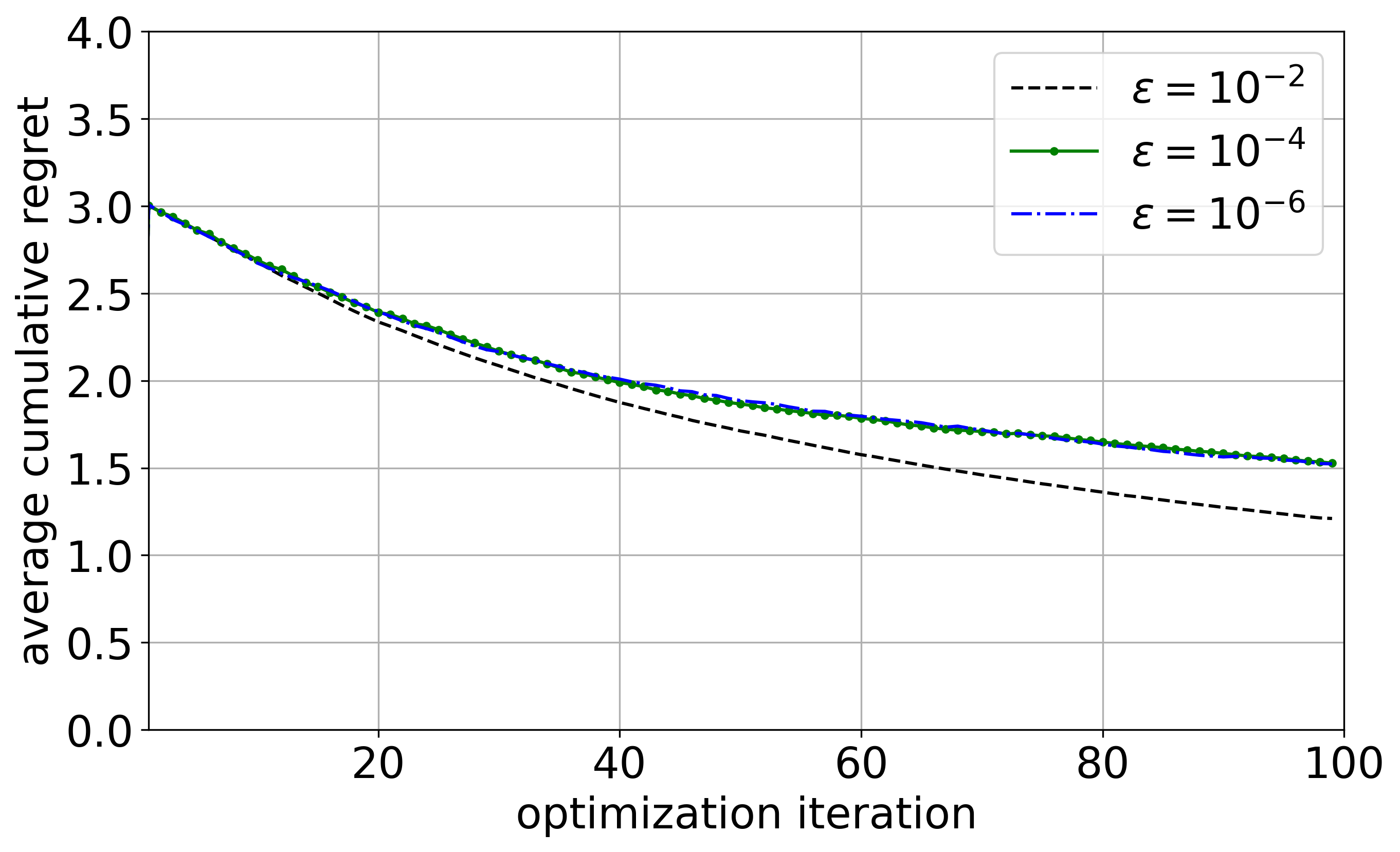}
    \label{fig:ex5}
  \end{subfigure}
    \begin{subfigure}{\linewidth}
    \centering
    \includegraphics[width=0.75\linewidth]{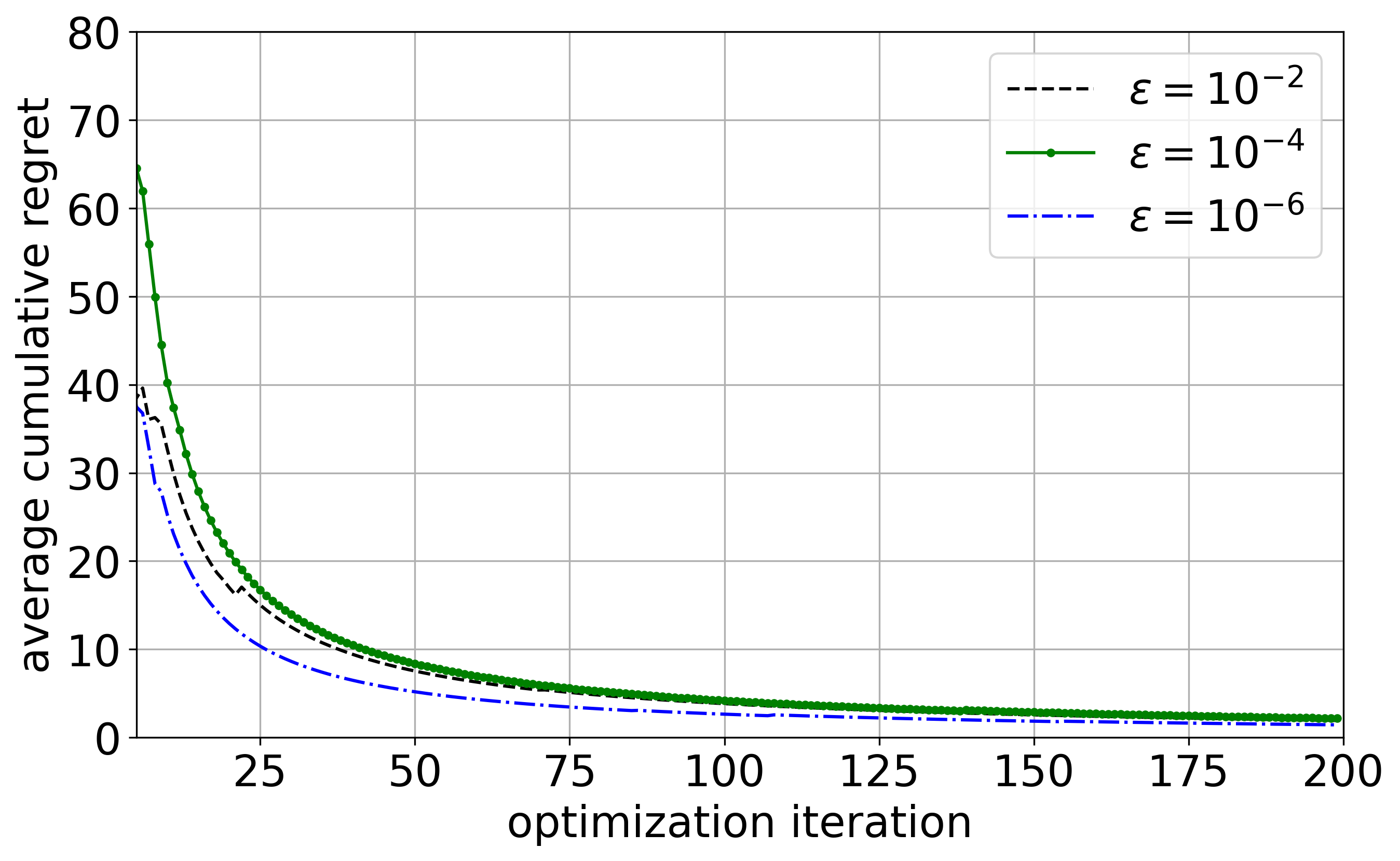}
    \label{fig:ex1}
\end{subfigure}
\begin{subfigure}{\linewidth}
    \centering
    \includegraphics[width=0.75\linewidth]{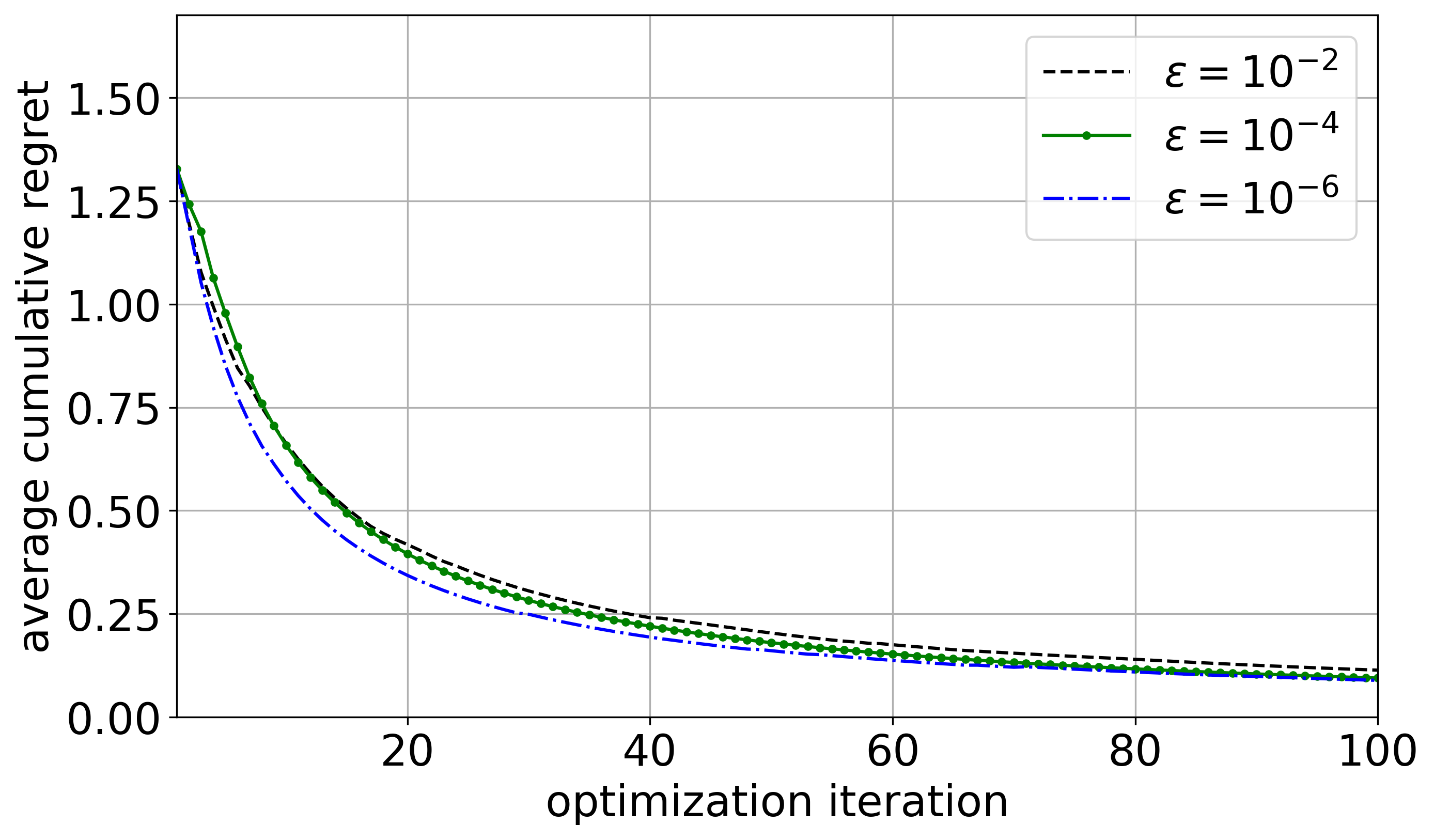}
  %
\end{subfigure}
  \caption{Median average cumulative regret bound for practical EGO with $\epsilon$ values $10^{-2}$, $10^{-4}$, and $10^{-6}$ for five examples. From top to bottom: Rosenbrock, Six-hump camel, Hartmann6, Branin, Michalewicz functions.}
  \label{fig:numerical}
\end{figure}

For all of our examples, practical EGO displays no-regret convergence behavior and is capable of finding the optimal solution rather efficiently, supporting our theoretical results. The different values of $\epsilon$ again align well with the regret bound theory, with $\epsilon=10^{-4}$ appearing to generate lowest regret in $3$ out of the $5$ examples. 

\section{Conclusions}\label{se:conclusion}
In this paper, we establish the novel instantaneous regret bound and the first cumulative regret upper bound for practical EGO, which is the default implementation of EGO in many available software packages. 
We show that it is a no-regret algorithm for kernels including SE and Matérn kernels. 
Our analysis thus provides cumulative regret theories on one of the most widely used BO algorithms.
Further, we provide theoretical guidelines on the choice of the nugget in that   $\epsilon
$ too large can lead to a worse cumulative regret upper bound. 
In practice, we anticipate the choice of $\epsilon$ to be influenced by other factors beyond our theoretical results such as the computational budget and the kernel.

\clearpage





\clearpage
\appendix 
\onecolumn
\section{Background}\label{appdx:back}
We first define an equivalent form of EI~\eqref{eqn:EI-1}.
We distinguish between its exploration and exploitation parts and define the \textit{trade-off} form $EI(a,b):\Rbb\times\Rbb\to\Rbb$ as
\begin{equation} \label{eqn:EI-ab}
 \centering
  \begin{aligned}
       EI(a,b) = a \Phi\left(\frac{a}{b}\right)+b \phi\left(\frac{a}{b}\right), 
  \end{aligned}
\end{equation}
where $b\in(0,1]$. One can view $a$ and $b$ as two independent variables.
For a given $\xbm$, if $a_t=f^+_{t}-\mu_{t}(\xbm)$ and $b_t=\sigma_{t}(\xbm)\in [0,1]$, then $EI(a_t,b_t)=EI_t(\xbm)$. 
Hence, we refer to $f^+_{t}-\mu_{t}(\xbm)$ and $\sigma_t(\xbm)$ the exploitation and exploration parts of $EI_t$, respectively.

The definition of RKHS is given below.
\begin{definition}\label{def:rkhs}
   Consider a positive definite kernel $k: \mathcal{X}\times \mathcal{X}\to\Rbb $ with respect to a finite Borel measure supported on $\mathcal{X}$. A Hilbert space $H_k$ of functions on $\mathcal{X}$ with an inner product $\langle \cdot,\cdot \rangle_{H_k}$ is called a RKHS with kernel $k$ if $k(\cdot,\xbm)\in H_k$ for all $\xbm\in \mathcal{X}$, and $\langle f,k(\cdot,\xbm)\rangle_{H_k}=f(\xbm)$ for all $\xbm\in \mathcal{X}, f\in H_k$. The induced RKHS norm $\norm{f}_{H_k}=\sqrt{\langle f,f\rangle_{H_k}}$ measures the smoothness of $f$ with respect to $k$.
\end{definition}

Union bound is given in the next lemma.
\begin{lemma}\label{lem:unionbound}
  For a countable set of events $A_1,A_2,\dots$, we have 
 \begin{equation*} \label{eqn:union-bound-1}
  \centering
  \begin{aligned}
    \Pbb(\bigcup_{i=1}^{\infty} A_i) \leq \sum_{i=1}^{\infty} \Pbb(A_i).
  \end{aligned}
\end{equation*}
\end{lemma}

To use the maximum information gain, we consider a Gaussian observation noise $\eta\sim\mathcal{N}(0,\epsilon)$. If such an i.i.d. noise exists, at sample point $\xbm_t$, we have the observation $y_t=f(\xbm_t)+\eta_t$. The maximum information gain can now be defined below.
\begin{definition}\label{def:infogain}
  Consider a set of sample points $A\subset C$. Given $\xbm_{A}$ and its function values $\fbm_A=[f(\xbm)]_{\xbm\in A}$, the mutual information between $\fbm_A$ and the observation $\ybm_A$ is $I(\ybm_A;\fbm_A)=H(\ybm_A)-H(\ybm_A|\fbm_A)$, where $H$ is the entropy. The maximum information gain $\gamma_T$ after $T$ samples is $\gamma_T = \max_{A\subset C,|A|=T} I(\ybm_A;\fbm_A)$. 
\end{definition}
Readers are referred to~\cite{cover1999elements,srinivas2009gaussian} for a detailed discussion of the maximum information gain.  
Here, we emphasize that practical EGO assumes no observation noise. However, we can continue to use maximum information gain as an analytic tool to bound $\sum_{t=1}^T  \sigma_{t-1}(\xbm_t)$. Indeed, $\sigma_{t-1}(\xbm)$ of practical EGO is the same as that of GP with Gaussian noise $\mathcal{N}(0,\epsilon)$. 

The following lemmas are well-established results from~\cite{srinivas2009gaussian} on the information gain and variances.
\begin{lemma}\label{lem:variancebound}
 The sum of posterior standard deviation at sample points  $\sigma_{t-1}(\xbm)$ satisfies
 \begin{equation} \label{eqn:var-1}
  \centering
  \begin{aligned}
    \sum_{t=1}^T  \sigma_{t-1}(\xbm_t) \leq \sqrt{C_{\gamma}(\epsilon) T \gamma_T(\epsilon)}, 
  \end{aligned}
\end{equation}
 where $C_{\gamma}(\epsilon) = 2/log(1+\epsilon^{-1})$.
\end{lemma}
Here, we emphasize that the maximum information gain is dependent on the nugget $\epsilon$. 
The state-of-the-art rates of $\gamma_t$ for two commonly used kernels are given below.
\begin{lemma}[\cite{iwazaki2025improved,vakili2021information}]\label{lem:gammarate}
  For a GP with $t$ samples, the SE kernel has $\gamma_t=\mathcal{O}(\log^{d+1}(t))$, and 
  the Matérn kernel with smoothness parameter $\nu>0$ has $\gamma_t=\mathcal{O}(t^{\frac{d}{2\nu+d}}\log^{\frac{2\nu}{2\nu+d}} (t))$.
\end{lemma}

Before concluding the section, we state the straightforward bound of $f$ on $C$ as a lemma for easy reference. 
\begin{lemma}\label{lem:f-bound}
  The function $f$ is bounded by $B$, \textit{i.e.}, $|f(\xbm)|\leq B$ for all $\xbm\in C$.
\end{lemma}
\begin{proof}
    From our assumption on the kernel $k(\xbm,\xbm) =1$, we can write 
    \begin{equation} \label{eqn:f-bound-pf-1}
  \centering
  \begin{aligned}
     |f(\xbm)|\leq \norm{f}_{H_k} k(\xbm,\xbm)\leq B.
  \end{aligned}
\end{equation}
\end{proof}

\section{Preliminary Results}\label{se:prep-proof}  
 
First, we state a property of $\tau(\cdot)$ in~\eqref{def:tau} below.
\begin{lemma}\label{lem:tau}
  The function $\tau(z)$ is monotonically increasing in $z$ and $\tau(z)>0$ for $\forall z\in \Rbb$.
  The derivative of $\tau(z)$ is $\Phi(z)$.
\end{lemma}
\begin{proof}
   From the definition of $\tau(z)$, we can write
  \begin{equation} \label{eqn:tau-1}
  \centering
   \begin{aligned}
   \tau(z)=z\Phi(z)+\phi(z) > \int_{-\infty}^{z} u\phi(u) du + \phi(z) = -\phi(u)|_{-\infty}^z+\phi(z) =0.
  \end{aligned}
  \end{equation}
  Given the definition of $\phi(u)$, 
  \begin{equation} \label{eqn:tau-2}
 \centering
  \begin{aligned}
      \frac{d \phi(u)}{d u} = \frac{1}{\sqrt{2\pi}} e^{-\frac{u^2}{2}} (-u) = -\phi(u)u.
   \end{aligned}
\end{equation}
   Thus, the derivative of $\tau$ is 
   \begin{equation} \label{eqn:tau-3}
 \centering
  \begin{aligned}
      \frac{d\tau(z)}{d z} = \Phi(z)+z\phi(z) -\phi(z)z = \Phi(z)>0.
   \end{aligned}
\end{equation}
\end{proof}
Another lemma on $\tau$ and $\Phi$ is given below.
\begin{lemma}\label{lem:tauvsPhi}
  Given $z>0$, $\Phi(-z)>\tau(-z)$. 
\end{lemma}
\begin{proof}
   Define $q(z) = \Phi(-z)-\tau(-z)$.
   Using integration by parts, we have
   \begin{equation} \label{eqn:tvp-pf-1}
 \centering
  \begin{aligned}
   \Phi(z) = \int_{-\infty}^z \phi(u)du > \int_{-\infty}^{z}\phi(u)\left(1-\frac{3}{u^4}\right)du=-\frac{\phi(z)}{z}+\frac{\phi(z)}{z^3}. 
   \end{aligned}
\end{equation}
  Replacing $z$ with $-z$ in~\eqref{eqn:tvp-pf-1},
   \begin{equation} \label{eqn:tvp-pf-1.5}
 \centering
  \begin{aligned}
   \phi(-z)\left(\frac{1}{z}-\frac{1}{z^3}\right)<\Phi(-z).
   \end{aligned}
\end{equation}
   Multiplying both sides in~\eqref{eqn:tvp-pf-1.5} by $1+z$,
\begin{equation} \label{eqn:tvp-pf-2}
 \centering
  \begin{aligned}
   (1+z)\Phi(-z)> \phi(-z)\frac{z^2-1}{z^3}(1+z)= \phi(-z)\left(1+\frac{z^2-z-1}{z^3}\right). 
   \end{aligned}
\end{equation}
   Thus, if $z > \frac{1+\sqrt{5}}{2}$, then the right-hand-side of~\eqref{eqn:tvp-pf-2} $> \phi(-z)$ and  
 \begin{equation} \label{eqn:tvp-pf-3}
 \centering
  \begin{aligned}
   q(z):=\Phi(-z)-\tau(-z) = (1+z)\Phi(-z) - \phi(-z)>0.
   \end{aligned}
\end{equation}
   Therefore, in the following, we focus on $z\in(0,\frac{1+\sqrt{5}}{2}]$. We analyze $q(z)$ using its derivatives.
   Taking the derivative of $q$, by Lemma~\ref{lem:tau},  
   \begin{equation} \label{eqn:tvp-pf-4}
 \centering
  \begin{aligned}
      \frac{d q(z)}{d z} = -\phi(-z)+\Phi(-z) := q'(z).
   \end{aligned}
\end{equation}
  Further, the derivative of $q'(z)$ is 
  \begin{equation} \label{eqn:tvp-pf-5}
 \centering
  \begin{aligned}
      \frac{d^2 q(z)}{d z^2} = \frac{d q'(z)}{d z} = -\phi(-z)+\phi(-z) z=\phi(z)(z-1).
   \end{aligned}
\end{equation}
  For $z>1$, $\frac{d^2 q(z)}{d z^2}>0$. For $0<z<1$, $\frac{d^2 q(z)}{d z^2}<0$. 
  Thus, $q'(z)$ is monotonically decreasing for $0<z<1$ and then monotonically increasing for $z>1$ .
  We first consider $q'(z)$ for $0<z<1$.
  We know that by simple algebra, $q'(0)=\Phi(0)-\phi(0)>0$ and $q'(1)=\Phi(-1)-\phi(-1)<0$.  
   Thus, there exists a $0<\bar{z}<1$ so that $q'(\bar{z})=0$.
  Next, for $z>1$, from Lemma~\ref{lem:tau}, we can write  
    \begin{equation} \label{eqn:tvp-pf-6}
 \centering
  \begin{aligned}
    q'(z) =  \Phi(-z)-\phi(-z) < z\Phi(-z) - \phi(-z) = - \tau(-z) < 0.
   \end{aligned}
\end{equation}
   Therefore, $0<\bar{z}<1<\frac{1+\sqrt{5}}{2}$ is a unique stationary point such that $q'(\bar{z})=0$.
   Thus, $q'(z)>0 $ for $0 <z< \bar{z}< \frac{1+\sqrt{5}}{2}$ and $q'(z)<0$ for $z> \bar{z}$.
   This means that for $0<z<\bar{z}$, $q(z)$ is monotonically increasing. For $ \bar{z}<z<\frac{1+\sqrt{5}}{2}$, $q(z)$ is monotonically decreasing. Therefore, $q(z) >\min\{q(0), q(\frac{1+\sqrt{5}}{2})\}$ for $z\in (0,\frac{1+\sqrt{5}}{2})$. Since $q(0)>0$ and $q(\frac{1+\sqrt{5}}{2})>0$, $q(z)>0$ for $z\in (0,\frac{1+\sqrt{5}}{2})$. 
   Combined with~\eqref{eqn:tvp-pf-3}, the proof is complete.
\end{proof}

The next lemma contains basic inequalities for $EI_t$.
\begin{lemma}\label{lem:EI}
 $EI_{t-1}(\xbm)$ satisfies 
    $EI_{t-1}(\xbm) \geq 0$ and $EI_{t-1}(\xbm) \geq f_{t-1}^+- \mu_{t-1}(\xbm)$.
Moreover, 
\begin{equation} \label{eqn:EI-property-2}
 \centering
  \begin{aligned}
   z_{t-1}(\xbm)\leq  \frac{EI_{t-1}(\xbm)}{\sigma_{t-1}(\xbm) }=\tau(z_{t-1}(\xbm)) < \begin{cases}  \phi(z_{t-1}(\xbm)), \ &z_{t-1}(\xbm)<0\\
             z_{t-1}(\xbm) +\phi(z_{t-1}(\xbm)), \ &z_{t-1}(\xbm)\geq 0.
      \end{cases}
  \end{aligned}
\end{equation}

\end{lemma}
\begin{proof}
     From the definition of $I_{t-1}$ and $EI_{t-1}$, the first statement follows immediately.
     By~\eqref{eqn:EI-1},  
\begin{equation} \label{eqn:EI-property-pf-1}
 \centering
  \begin{aligned}
     \frac{EI_{t-1}(\xbm)}{\sigma_{t-1}(\xbm) } = z_{t-1}(\xbm) \Phi(z_{t-1}(\xbm)) + \phi(z_{t-1}(\xbm)). 
  \end{aligned}
\end{equation}
If $z_{t-1}(\xbm)<0$, or equivalently $f_{t-1}^+ -\mu_{t-1}(\xbm)<0$,~\eqref{eqn:EI-property-pf-1} leads to
     $\frac{EI_{t-1}(\xbm)}{\sigma_{t-1}(\xbm)} < \phi(z_{t-1}(\xbm))$. 
If $z_{t-1}(\xbm)\geq 0$, $\Phi(\cdot)<1$ gives us
     $\frac{EI_{t-1}(\xbm)}{\sigma_{t-1}(\xbm)} < z_{t-1}(\xbm) +\phi(z_{t-1}(\xbm))$. 
 The left inequality in~\eqref{eqn:EI-property-2} is an immediate result of $EI_{t-1}(\xbm) \geq f^+_{t-1}- \mu_{t-1}(\xbm)$.
\end{proof}
 The monotonicity of the exploration and exploitation of~\eqref{eqn:EI-ab} is given next, previously also shown in~\cite{jones1998efficient}.
\begin{lemma}\label{lem:EI-ms}
  $EI(a,b)$ is monotonically increasing in both $a$ and $b$ for $b\in(0,1]$. 
\end{lemma}
\begin{proof}
   We prove the lemma by taking the derivative of $EI(a,b)$ with respect to both variables. First,  
 \begin{equation} \label{eqn:EI-ms-pf-1}
 \centering
  \begin{aligned}
   \frac{\partial EI(a,b)}{\partial a} = \Phi\left(\frac{a}{b}\right) + a\phi\left(\frac{a}{b}\right) \frac{1}{b} + b\frac{\partial \phi\left(\frac{a}{b}\right)}{\partial a}.
   \end{aligned}
\end{equation}
   From~\eqref{eqn:tau-2},~\eqref{eqn:EI-ms-pf-1} is  
  \begin{equation} \label{eqn:EI-ms-pf-2}
 \centering
  \begin{aligned}
   \frac{\partial EI(a,b)}{\partial a} = \Phi\left(\frac{a}{b}\right) + \phi\left(\frac{a}{b}\right) \frac{a}{b}  -\phi\left(\frac{a}{b}\right)\frac{a}{b}  = \Phi\left(\frac{a}{b}\right)>0.
   \end{aligned}
\end{equation}
  Similarly, 
 \begin{equation} \label{eqn:EI-ms-pf-3}
 \centering
  \begin{aligned}
   \frac{\partial EI(a,b)}{\partial b} =&  -a\phi\left(\frac{a}{b}\right)\frac{a}{b^2} + \phi\left(\frac{a}{b}\right)-b\phi\left(\frac{a}{b}\right)\frac{a}{b}(-\frac{a}{b^2})
     = \phi\left(\frac{a}{b}\right)>0.
   \end{aligned}
\end{equation}
  
\end{proof}

The next lemma puts a lower bound on $f^+_{t-1}-\mu_{t-1}(\xbm)<0$ if $EI_{t-1}(\xbm)$ is bounded below by a positive sequence denoted as $\kappa_t$. It is also previously shown in~\cite{nguyen17a}.
\begin{lemma}\label{lem:mu-bounded-EI}
   If $EI_{t-1}(\xbm)\geq \kappa_t$ for some $\kappa_t \in(0,\frac{1}{\sqrt{2\pi}})$ and $f^+_{t-1}-\mu_{t-1}(\xbm)<0$, then we have 
  \begin{equation} \label{eqn:mu-bounded-EI-1}
  \centering
  \begin{aligned}
           f_{t-1}^+-\mu_{t-1}(\xbm) \geq -  \sqrt{2\log\left(\frac{1}{\sqrt{2\pi}\kappa_t}\right)}\sigma_{t-1}(\xbm).
   \end{aligned}
\end{equation}
\end{lemma}
\begin{proof}
     By definition of $EI_{t-1}(\xbm)$,
  \begin{equation} \label{eqn:mu-bounded-EI-pf-1}
  \centering
  \begin{aligned}
       \kappa_t\leq& (f_{t-1}^+-\mu_{t-1}(\xbm))\Phi(z_{t-1}(\xbm))+\sigma_{t-1}(\xbm)\phi(z_{t-1}(\xbm)) < \sigma_{t-1}(\xbm)\phi(z_{t-1}(\xbm))\\
    =& \sigma_{t-1}(\xbm) \frac{1}{\sqrt{2\pi}} e^{-\frac{1}{2}z_{t-1}^2(\xbm)}\leq \frac{1}{\sqrt{2\pi}} e^{-\frac{1}{2}z_{t-1}^2(\xbm)}.
   \end{aligned}
\end{equation}
   Rearranging and taking the logarithm of~\eqref{eqn:mu-bounded-EI-pf-1}, we have 
  \begin{equation} \label{eqn:mu-bounded-EI-pf-2}
  \centering
  \begin{aligned}
       2\log\left(\frac{1}{\sqrt{2\pi}\kappa_t}\right) >  z_{t-1}^2(\xbm) = \left(\frac{f_{t-1}^+-\mu_{t-1}(\xbm)}{\sigma_{t-1}(\xbm)}\right)^2.
   \end{aligned}
\end{equation}
   Given that $f^+_{t-1}-\mu_{t-1}(\xbm)<0$, we recover~\eqref{eqn:mu-bounded-EI-1}.  
\end{proof}

A global lower bound of the posterior variance is given in the next lemma.
\begin{lemma}\label{lemma:gp-sigma-bound}
    The posterior standard deviation~\eqref{eqn:GP-post} has the lower bound  
 \begin{equation} \label{eqn:gp-sigma-bound-1}
 \centering
  \begin{aligned}
     \sigma_t(\xbm) \geq   \sqrt{\frac{\epsilon}{t+\epsilon}}.
   \end{aligned}
\end{equation}
\end{lemma}
\begin{proof}
    We invoke the fact that the minimum posterior standard deviation at $\xbm$ is obtained if the previous $t$ samples are all $\xbm$. In this case, all entries of $\Kbm_t$ are $1$. It is easy to verify that  
   \begin{equation} \label{eqn:gp-sigma-bound-pf-2}
 \centering
  \begin{aligned}
      (\Kbm_t+\epsilon\Ibm)^{-1} = -\frac{1}{t\epsilon+\epsilon^2}\Pbm +\frac{1}{\epsilon} \Ibm,
  \end{aligned}
\end{equation}
where $\Pbm$ is a $t\times t$ matrix with all entries being $1$.
Thus, by~\eqref{eqn:GP-post}, we have 
   \begin{equation} \label{eqn:gp-sigma-bound-pf-5}
 \centering
  \begin{aligned}
    \sigma_t^2(\xbm) \geq  1- \pbm^T  \left[-\frac{1}{t+\epsilon}\Pbm +\frac{1}{\epsilon} \Ibm\right] \pbm = \frac{\epsilon}{t+\epsilon},
  \end{aligned}
\end{equation}
   where $\pbm$ is the $t$-dimensional vector with all $1$ entries.
\end{proof}

Next, we present a well-established bound on $f$ and the prediction $\mu_{t-1}(\xbm)$.
\begin{lemma}\label{lem:fmu}
For any given $\xbm\in C$ and $t\geq 1$,   
\begin{equation} \label{eqn:fmu-1}
 \centering
  \begin{aligned}
     |f(\xbm) - \mu_{t-1}(\xbm)  | \leq B \sigma_{t-1}(\xbm). 
  \end{aligned}
\end{equation}
\end{lemma}
The proof of Lemma~\ref{lem:fmu} can be found in Theorem 2 of~\cite{chowdhury2017kernelized}.
Next, we extend the bounds to $|I_{t-1}(\xbm)-EI_{t-1}(\xbm)|$.
An upper bound of $I_{t-1}(\xbm)-EI_{t-1}(\xbm)$ is given in the next lemma.
\begin{lemma}\label{lem:fandnu}
For any given $\xbm\in C$, $t\in\Nbb$, and $w>0$,   
\begin{equation} \label{eqn:fandmu-1}
 \centering
  \begin{aligned}
    I_{t-1}(\xbm) - EI_{t-1}(\xbm)  < \begin{cases} \sigma_t(\xbm) w, \  &f_{t-1}^+- f(\xbm)\leq 0\\
                     -f(\xbm) +\mu_{t-1}(\xbm),  \ & f^+_{t-1}-f(\xbm)>0.
         \end{cases}
  \end{aligned} 
\end{equation}
\end{lemma}
\begin{proof}
      If $f^+_{t-1}-f(\xbm)\leq 0$, we have by definition~\eqref{eqn:improvement} and Lemma~\ref{lem:EI}, 
\begin{equation} \label{eqn:fandmu-pf-1}
 \centering
  \begin{aligned}
     I_{t-1}(\xbm)-EI_{t-1}(\xbm) =  - EI_{t-1}(\xbm)  < 0< \sigma_{t-1}(\xbm) w. 
  \end{aligned}
\end{equation}
For $f^+_{t-1}- f(\xbm)>0$, we can write via Lemma~\ref{lem:EI}, 
\begin{equation} \label{eqn:fandmu-pf-2}
 \centering
  \begin{aligned}
   I_{t-1}(\xbm)-EI_{t-1}(\xbm) =& f^+_{t-1}- f(\xbm) - EI_{t-1}(\xbm)\\
   \leq&  f^+_{t-1}-f(\xbm)-f^+_{t-1}+\mu_{t-1}(\xbm)  < -f(\xbm)+\mu_{t-1}(\xbm). 
  \end{aligned}
\end{equation}
\end{proof}

An upper bound on $I_{t-1}(\xbm)$ and $EI_{t-1}(\xbm)$ is given in the next lemma.
\begin{lemma}\label{lem:IEI-bound-ratio}
 The improvement function and $EI$ function satisfy  
 \begin{equation} \label{eqn:IEI-bound-ratio-1}
  \centering
  \begin{aligned}
     I_{t-1}(\xbm) \leq \frac{\tau(B)}{\tau(-B)} EI_{t-1}(\xbm), \forall \xbm\in C, \forall t\in\Nbb.
  \end{aligned}
\end{equation}
\end{lemma}
\begin{proof}
We consider two cases. First,
   if $f^+_{t-1}-f(\xbm) \leq 0$, then $I_{t-1}(\xbm) = 0$.
   Since $EI_{t-1}(\xbm)\geq 0$,~\eqref{eqn:IEI-bound-ratio-1} stands.

   Second, if  $f^+_{t-1}-f(\xbm) > 0$, then
 \begin{equation} \label{eqn:IEI-bound-ratio-pf-1}
 \centering
  \begin{aligned}
     f^+_{t-1}-\mu_{t-1}(\xbm) &=  f^+_{t-1}-f(\xbm)+f(\xbm)-\mu_{t-1}(\xbm) >f(\xbm)-\mu_{t-1}(\xbm).\\
  \end{aligned}
 \end{equation}
   From the one-side inequality in Lemma~\ref{lem:fmu},~\eqref{eqn:IEI-bound-ratio-pf-1} implies
 \begin{equation} \label{eqn:IEI-bound-ratio-pf-1.2}
 \centering
  \begin{aligned}
    f^+_{t-1}-\mu_{t-1}(\xbm) >-B\sigma_{t-1}(\xbm).
  \end{aligned}
\end{equation}
   Then, from Lemma~\ref{lem:tau} the monotonicity of $\tau(\cdot)$, we have
 \begin{equation} \label{eqn:IEI-bound-ratio-pf-1.5}
 \centering
  \begin{aligned}
     \tau\left(z_{t-1}(\xbm)\right) >\tau(-B),
  \end{aligned}
\end{equation}
  where $z_{t-1}(\xbm)=\frac{f^+_{t-1}-\mu_{t-1}(\xbm)}{\sigma_{t-1}(\xbm)}$. Since  $EI_{t-1}(\xbm) = \sigma_{t-1}(\xbm) \tau(z_{t-1}(\xbm))$, we can write
  \begin{equation} \label{eqn:IEI-bound-ratio-pf-2}
 \centering
  \begin{aligned}
     EI_{t-1}(\xbm) = \sigma_{t-1}(\xbm)\tau\left(z_{t-1}(\xbm)\right) >\tau(-B)\sigma_{t-1}(\xbm).
  \end{aligned}
\end{equation}
  Next, we let $w=B$ in Lemma~\ref{lem:fandnu} and obtain
  \begin{equation} \label{eqn:IEI-bound-ratio-pf-4}
  \centering
  \begin{aligned}
       I_{t-1}(\xbm)- EI_{t-1}(\xbm) \leq -f(\xbm) +\mu_{t-1}(\xbm)\leq B\sigma_{t-1}(\xbm).
  \end{aligned}
\end{equation}
  Applying~\eqref{eqn:IEI-bound-ratio-pf-4} to~\eqref{eqn:IEI-bound-ratio-pf-2} by eliminating $\sigma_{t-1}(\xbm)$ and using union bound, we have
 \begin{equation} \label{eqn:IEI-bound-ratio-pf-5}
 \centering
  \begin{aligned}
     EI_{t-1}(\xbm)  > \frac{\tau(-B)}{B+\tau(-B)} I_{t-1}(\xbm) = \frac{\tau(-B)}{\tau(B)} I_{t-1}(\xbm).
  \end{aligned}
\end{equation}
\end{proof}
A similar result to Lemma~\ref{lem:IEI-bound-ratio} is previously shown in~\cite{bull2011convergence}.

\section{Instantaneous Regret Bound Proof}\label{se:inst-regret-proof}

Proof of Lemma~\ref{lem:ego-instregret-1} is given next.
\begin{proof}
  We consider two cases based on the value of $f^+_{t-1}-f(\xbm_t)$. First, $ f^+_{t-1} \leq f(\xbm_{t})$.
  From Lemma~\ref{lem:fmu} and~\ref{lem:IEI-bound-ratio},
   \begin{equation} \label{eqn:ego-instregret-pf-1}
  \centering
  \begin{aligned}
          r_t  =& f(\xbm_t) - f(\xbm^*) =f(\xbm_t) - f^+_{t-1} + f^+_{t-1}  - f(\xbm^*)\leq f(\xbm_t) - f^+_{t-1} + I_{t-1}(\xbm^*)\\ 
          \leq& f(\xbm_t) -\mu_{t-1}(\xbm_t)+\mu_{t-1}(\xbm_t)- f^+_{t-1}+\frac{\tau(B)}{\tau(-B)}EI_{t-1}(\xbm^*) \\ 
          \leq&  \mu_{t-1}(\xbm_t) - f^+_{t-1}+c_B EI_{t-1}(\xbm_{t})  + B \sigma_{t-1}(\xbm_t),\\ 
  \end{aligned}
  \end{equation}
   where $c_B=\frac{\tau(B)}{\tau(-B)}$.

  Next, we aim to provide a lower bound for $EI_{t-1}(\xbm_t)$ by comparing $EI_{t-1}(\xbm_t)$ to $EI_{t-1}(\xbm^*)$.
  For $EI_{t-1}(\xbm^*)$, we can write via Lemma~\ref{lem:fmu} that  
  \begin{equation} \label{eqn:ego-instregret-pf-2}
 \centering
 \begin{aligned}
   &EI_{t-1}(\xbm^*) = \sigma_{t-1}(\xbm^*)\tau(z_{t-1}(\xbm^*)) = \sigma_{t-1}(\xbm^*)\tau\left(\frac{f^+_{t-1}-f(\xbm^*)+f(\xbm^*)-\mu_{t-1}(\xbm^*)}{\sigma_{t-1}(\xbm^*)}\right)\\
      \geq& \sigma_{t-1}(\xbm^*)\tau\left(\frac{f^+_{t-1}-f(\xbm^*)}{\sigma_{t-1}(\xbm^*)} - B\right)
      \geq \sigma_{t-1}(\xbm^*)\tau\left( - B\right),
   \end{aligned}
  \end{equation}
where the second inequality uses $f(\xbm^*)\leq f(\xbm)$.
  By definition $EI_{t-1}(\xbm_t)\geq EI_{t-1}(\xbm^*)$. By Lemma~\ref{lemma:gp-sigma-bound},~\eqref{eqn:ego-instregret-pf-2} implies  
   \begin{equation} \label{eqn:ego-instregret-pf-3}
 \centering
 \begin{aligned}
     EI_{t-1}(\xbm_t) \geq \tau(-B) \sqrt{\frac{\epsilon}{t+\epsilon}}.
   \end{aligned}
  \end{equation}
 It is easy to see that $\tau(-B)\sqrt{\frac{\epsilon}{t+\epsilon}}<\frac{1}{\sqrt{2\pi}}$.  By Lemma~\ref{lem:mu-bounded-EI},~\eqref{eqn:ego-instregret-pf-3} implies 
\begin{equation} \label{eqn:ego-instregret-pf-4}
  \centering
  \begin{aligned}
          f^+_{t-1}-\mu_{t-1}(\xbm_t) \geq - \log^{\frac{1}{2}}\left(\frac{t+\epsilon}{2\pi \tau^2(-B)  \epsilon}  \right) \sigma_{t-1}(\xbm_t).
  \end{aligned}
  \end{equation}
  Equivalently,
\begin{equation} \label{eqn:ego-instregret-pf-5}
  \centering
  \begin{aligned}
         \mu_{t-1}(\xbm_t)- f^+_{t-1}\leq c_{B\epsilon}(\epsilon,t)  \sigma_{t-1}(\xbm_t).
  \end{aligned}
  \end{equation}
  where $c_{B\epsilon}(\epsilon,t)= \log^{\frac{1}{2}}(\frac{t+\epsilon}{2\pi \tau^2(-B) \epsilon})$.  
  Using Lemma~\ref{lem:fmu}, $\Phi(\cdot)<1$, and $f^+_{t-1}-f(\xbm_{t})\leq 0$, we have 
  \begin{equation} \label{eqn:ego-instregret-pf-6}
 \centering
 \begin{aligned}
          EI_{t-1}(\xbm_t) 
          =& (f^+_{t-1}-\mu_{t-1}(\xbm_{t}))\Phi(z_{t-1}(\xbm_{t}))+\sigma_{t-1}(\xbm_{t})\phi(z_{t-1}(\xbm_{t}))\\ 
          \leq& (f^+_{t-1}-f(\xbm_{t})+f(\xbm_{t})-\mu_{t-1}(\xbm_{t}))\Phi(z_{t-1}(\xbm_{t}))+\phi(0)\sigma_{t-1}(\xbm_{t})\\ 
          \leq& B\sigma_{t-1}(\xbm_{t})+\phi(0)\sigma_{t-1}(\xbm_{t}). 
  \end{aligned}
  \end{equation}
  Applying\eqref{eqn:ego-instregret-pf-6} and~\eqref{eqn:ego-instregret-pf-5}  to~\eqref{eqn:ego-instregret-pf-1}, we have 
   \begin{equation} \label{eqn:ego-instregret-pf-7}
  \centering
  \begin{aligned}
          r_t  
          \leq&  (c_{B\epsilon}(\epsilon,t) + B+ c_B(B+\phi(0)))\sigma_{t-1}(\xbm_t).\\ 
  \end{aligned}
  \end{equation}

   Second, if $f^+_{t-1}-f(\xbm_{t})\geq 0$, we have  
   \begin{equation} \label{eqn:ego-instregret-pf-8}
  \centering
  \begin{aligned}
        r_t =& f(\xbm_{t})-f(\xbm^*) =  f(\xbm_{t})-f^+_{t-1}+ f^+_{t-1}-f(\xbm^*)\\
              \leq& f(\xbm_{t})-f^+_{t-1}+c_B EI_{t-1}(\xbm^*) \\
              \leq& f(\xbm_{t})-f^+_{t-1}+c_BEI_{t-1}(\xbm_{t}).\\
  \end{aligned}
  \end{equation}
  The first inequality in~\eqref{eqn:ego-instregret-pf-8} is by Lemma~\ref{lem:IEI-bound-ratio}.
  Further, we can write 
  \begin{equation} \label{eqn:ego-instregret-pf-9}
 \centering
 \begin{aligned}
          EI_{t-1}(\xbm_t) 
          =& (f^+_{t-1}-\mu_{t-1}(\xbm_{t}))\Phi(z_{t-1}(\xbm_{t}))+\sigma_{t-1}(\xbm_{t})\phi(z_{t-1}(\xbm_{t}))\\ 
          \leq& (f^+_{t-1}-f(\xbm_{t})+f(\xbm_{t})-\mu_{t-1}(\xbm_{t}))\Phi(z_{t-1}(\xbm_{t}))+\phi(0)\sigma_{t-1}(\xbm_{t})\\ 
          \leq& f^+_{t-1}-f(\xbm_{t})+ B\sigma_{t-1}(\xbm_{t})+\phi(0)\sigma_{t-1}(\xbm_{t}). 
  \end{aligned}
  \end{equation}
   Applying~\eqref{eqn:ego-instregret-pf-9} to~\eqref{eqn:ego-instregret-pf-8}, we have  
   \begin{equation} \label{eqn:ego-instregret-pf-10}
  \centering
  \begin{aligned}
        r_t  \leq& (c_B-1) (f_{t-1}^+-f(\xbm_t))+c_B (B+\phi(0))\sigma_{t-1}(\xbm_{t}).
  \end{aligned}
  \end{equation}
   Combine~\eqref{eqn:ego-instregret-pf-7} and~\eqref{eqn:ego-instregret-pf-10} and we have 
\begin{equation} \label{eqn:ego-instregret-pf-11}
  \centering
  \begin{aligned}
        r_t  \leq& \max\{c_B-1,0\} \max\{f_{t-1}^+-f(\xbm_t),0\}+(c_{B\epsilon}(\epsilon,t)+B+c_B(B+\phi(0)))\sigma_{t-1}(\xbm_{t}).\\
  \end{aligned}
  \end{equation}
 \end{proof}

\section{Cumulative Regret Bound Proof}\label{se:cumu-regret-proof}

The proof of Lemma~\ref{lemma:boi-inst-regret} is given next.
\begin{proof}
  From Lemma~\ref{lem:ego-instregret-1}, we consider the term $\sum_{t=1}^T  \max\{f_{t-1}^+-f(\xbm_{t}),0\}$. 
   Let $P_T\subseteq \{1,\dots,T\}$ be the ordered index set such that $f_{t-1}^+-f(\xbm_{t})>0$. 
   Then, using $t_i-1\geq t_{i-1}$, we have 
     \begin{equation} \label{eqn:boi-inst-regret-pf-1}
  \centering
  \begin{aligned}
     &  \sum_{t=1}^T  \max\{f_{t-1}^+-f(\xbm_t),0\}  =  \sum_{i=1}^{|P_T|} f^+_{t_i-1}-f(\xbm_{t_i})\leq  \sum_{i=1}^{|P_T|} f(\xbm_{t_{i-1}}) - f(\xbm_{t_i}),
    \end{aligned}
  \end{equation}
  where $t_i\in P_T$ and $t_i<t_{i+1}$.
Since $f_t^+\leq f(\xbm_t)$ for $\forall t\in\Nbb$,~\eqref{eqn:boi-inst-regret-pf-1} leads to 
\begin{equation} \label{eqn:boi-inst-regret-pf-2}
  \centering
  \begin{aligned}
       \sum_{t=1}^T  \max\{f_{t-1}^+-f(\xbm_t),0\}  \leq&  \sum_{i=1}^{|P_T|} f(\xbm_{t_{i-1}}) - f(\xbm_{t_i})\\
       &\leq  (f_{t_0}-f_{t_{|P_T|}}) 
       \leq  2B,
    \end{aligned}
  \end{equation}
  where we used the boundedness of $f$ in Lemma~\ref{lem:f-bound}.
  Using~\eqref{eqn:boi-inst-regret-pf-2} in~\eqref{eqn:ego-instregret-1}, we have 
     \begin{equation} \label{eqn:boi-inst-regret-pf-9}
  \centering
  \begin{aligned}
    R_T= \sum_{t=1}^T r_t \leq& \sum_{t=1}^T c_{B1} \max\{f_{t-1}^+-f(\xbm_{t}),0\}+\sum_{t=1}^T (c_{B\epsilon}(T)+B+c_B (B+\phi(0))) \sigma_{t-1}(\xbm_t)   \\
     \leq& 2c_{B1} B +  (c_{B\epsilon}(\epsilon,T) +B+c_B(B+\phi(0)))\sum_{t=1}^T  \sigma_{t-1}(\xbm_t).   \\
    \end{aligned}
  \end{equation}
\end{proof}

Proof of Theorem~\ref{thm:boi-cumulative-regret} is given next.
\begin{proof}
    From Lemma~\ref{lem:variancebound}, we know $\sum_{t=1}^T\sigma_{t-1}(\xbm_t) = \mathcal{O}(\sqrt{T\gamma_T})$. 
     From Lemma~\ref{lemma:boi-inst-regret}, the cumulative regret bound is 
    \begin{equation} \label{eqn:ei-boi-cumu-pf-2}
  \centering
  \begin{aligned}
         \mathcal{O}(R_T)=&\mathcal{O} ( \log^{1/2}(T) \sqrt{T\gamma_T}).
    \end{aligned}
  \end{equation} 
  Using $\gamma_T=\mathcal{O}(\log^{d+1}(T))$ for SE kernel, we have $R_T=\mathcal{O}(T^{\frac{1}{2}} \log^{\frac{d+2}{2}}(T))$.
  Using $\gamma_T=\mathcal{O}(T^{\frac{d}{2\nu+d}}\log^{\frac{2\nu}{2\nu+d}}(T))$ for Matérn kernel from~\cite{iwazaki2025improved,vakili2021information}, we have $R_T=\mathcal{O}(T^{\frac{\nu+d}{2\nu+d}}\log^{\frac{2\nu+0.5d}{2\nu+d}}(T))$.

\end{proof}

\section{Proof of Nugget Effect}\label{se:proofnugget}
The maximum information gain bounds are given below from~\cite{iwazaki2025improved}.
We use the following lemma on the upper bound of $\gamma_T$. 
\begin{lemma}[Maximum information gain upper bound, Theorem 7 in~\cite{iwazaki2025improved}]\label{thm:MIG} 
Assume the domain satisfies 
$C = \{\xbm\in\mathbb{R}^{d}, \|\xbm\|_{2}\le 1\}$. At given $d$ and $T$, for SE kernel, if $\theta\le e^{2}c_{d}\text{ and }T \geq (e-1)\epsilon $, 
 \begin{equation} \label{eqn:shogo-1}
  \centering
  \begin{aligned}
     \gamma_T(\epsilon) \leq \frac{C^{1}_{d}}{\theta^{d}} 
          \log^{d+1}\left(1+T/\epsilon\right)
          +\log \left(1+T/\epsilon\right)
          +C^{2}_{d} \exp\left(-\frac{2}{\theta}+\frac{1}{\theta^{2}}\right), \\
    \end{aligned}
  \end{equation}
 Further, for $\theta > e^{2}c_{d}$, 
  \begin{equation} \label{eqn:shogo-2}
  \centering
  \begin{aligned}
        \gamma_{T}(\epsilon) \leq 
          \frac{C^{3}_{d}}{\log^{d}\bigl(\frac{\theta}{e\,c_{d}}\bigr)}
          \log^{d+1}\left(1+\frac{T}{\epsilon}\right)
          +
          C^{4}_{d}\log \left(1+\frac{T}{\epsilon}\right)
          +C^{5}_{d},
     \end{aligned}
  \end{equation}
  where $\theta=2l^{2}$ and $c_{d}=\max\left\{1,
      \exp \left( \frac{1}{e}\!(\frac{d}{2}-1)\right)\right\}$.
      The constants $C_d^i,i=1,\dots,5$ only depend on $d$.
      
For Matérn kernels with smoothness $\nu>\frac12$, we have   
\begin{equation} \label{eqn:shogo-matern}
  \centering
  \begin{aligned}
     \gamma_{T}(\epsilon) \leq 
      C(T,\nu,\epsilon) \bar{\gamma}_{T} + C,
    \end{aligned}
  \end{equation}
where $C_\nu>0$ only depends on $\nu$, $C$ is a constant, and 
\begin{equation} \label{eqn:shogo-matern-2}
  \centering
  \begin{aligned}
     C(T,\nu,\lambda) =  \max \left\{
          1, {\log_{2}}\!\left(
            1+\frac{\Gamma(\nu)}{C_\nu}\,
            \log\left(\frac{T^2}{\epsilon}\right)
          \right)
          +\frac{1}{\nu}
           \log_{2} \left(
            \frac{T^2}{\nu\Gamma(\nu)\epsilon}
          \right)+1
        \right\},
    \end{aligned}
  \end{equation}
     Furthermore, 
     \begin{equation} \label{eqn:shogo-matern-3}
  \centering
  \begin{aligned}
        \bar{\gamma}_{T}(\epsilon) =  
        C^{1}_{d,\nu}
        \log\left(1+\frac{2T}{\epsilon}\right)
          +C^{2}_{d,\nu}
        \left(\frac{T}{\epsilon l^{2\nu}}\right)^{\frac{d}{2\nu+d}}
        \log^{\frac{2\nu}{2\nu+d}}\left(1+\frac{2T}{\epsilon}\right).
      \end{aligned}
  \end{equation}
  The constants $C^{i}_{d,\nu}$ depend only on $d$ and $\nu$. 
\end{lemma}
 Given this premise, we rewrite Lemma~\ref{thm:MIG}  into the following two lemmas on the bounds of $\gamma_T$. 
\begin{lemma}\label{lem:mig-se}
    Assume the domain satisfies 
$C = \{\xbm\in\mathbb{R}^{d}, \|\xbm\|_{2}\le 1\}$. For a SE kernel with fixed $l$,  there exist constants $C_{dl}^1$, $C_{dl}^2$, and $C_{dl}^3$ such that 
 \begin{equation} \label{eqn:shogo-se}
  \centering
  \begin{aligned}
     \gamma_T(\epsilon) \leq s_T(\epsilon)= C^{1}_{dl} 
         \left[ \log^{d+1}(1+T/\epsilon)
           +C^{2}_{dl} \log(1+T/\epsilon)+C^{3}_{dl}\right].  \\
    \end{aligned}
  \end{equation}
where the constants depend only on given $d$ and $l$.
\end{lemma}
\begin{lemma} \label{lem:mig-matern-shogo}
Assume the domain satisfies the condition in Theorem 7 in [1]. For a Matérn kernel with given $l$, $\nu>1/2$ and $d$, suppose $T/\epsilon$ is large enough such that  
\begin{equation} \label{eqn:shogo-matern-1}
  \centering
  \begin{aligned}
     c_T^0(\epsilon) :=  C_{\nu}^1[ \log(
            1+C_{\nu}^2
            \log\left(T^2 /\epsilon\right)
          )
          +C_{\nu}^3
           \log (T^2/\epsilon
          )+C_{\nu}^4]\geq 1,
    \end{aligned}
  \end{equation}
  where $C_{\nu}^1=\frac{1}{\log(2)}$, $C_{\nu}^2=\frac{\Gamma(\nu)}{C_{\nu}}$, $C_{\nu}^3=\frac{1}{\nu}$,$C_{\nu}^4=\frac{1}{\nu}\log(\frac{1}{\nu\Gamma(\nu)})+\log(2)$,  $C_{\nu}>0$ only depend on $\nu$ and $\Gamma$ is the Gamma function. Notice that $C_{\nu}^i>0,i=1,\dots,3$.
  Further, there exist $C^{1}_{d\nu l}$ and $C^{2}_{d\nu l}$ dependent only on $d$, $\nu$ and $l$ so that 
     $$
        \bar{\gamma}_{T}(\epsilon) =  
        C^{1}_{d\nu l}[
        \log(1+2T/\epsilon)+
        C^{2}_{d\nu l}\left(T\eta\right)^{\frac{d}{2\nu+d}} \log^{\frac{2\nu}{2\nu+d}}(1+2T/\epsilon)],
     $$
     where $C_{d\nu l}^1 = C_{d,\nu}^{(1)}$ and $C_{d\nu l}^2 = ( C_{d,\nu}^{(2)}(l)^{-2\nu d/(2\nu+d)})/(C_{d,\nu}^{(1)})$ (see [1] for constants $C_{d,\nu}^{(i)}$, $i=1,2$).
Then, there exists constants $C>0$ such that its $\gamma_T$ satisfies 
$$ \gamma_T(\epsilon) \leq s_T(\epsilon) =c_T^0(\epsilon)\bar{\gamma}_T(\epsilon) + C.$$
\end{lemma}
Note  that $C_{\nu}^4$ is not necessarily positive. 
The assumption that $c_T^0(\epsilon)\geq 1$ can be achieved when $T/\epsilon\gg 1$.We shall further assume that $T/\epsilon$ is sufficiently large such that $c_T^0(\epsilon)\bar{\gamma}_T(\epsilon)\geq C$.

Next, we simplify the upper bound of $R_T(\epsilon)$.
\begin{lemma}\label{lem:cumu-regret-bound-simple}
 Let  $C_R^1=2c_{B1} B$, $C_{R}^2 = \log\left(\frac{1}{2\pi\tau^2(-B)}\right)$, $C_R^3= B+c_B(B+\phi(0))$, and $C_R^4=C_R^2+(C_R^3)^2$ (see Lemma~\ref{lemma:boi-inst-regret} for constants). Suppose the upper bound on $\gamma_T(\epsilon)$ is $\gamma_T(\epsilon)\leq s_T(\epsilon)$. Define 
 \begin{equation*} \label{def:regret-bound-se}
  \centering
  \begin{aligned}
  c_T(\epsilon) =   \left(\log(1+T/\epsilon)+2C_R^3 (\log(1+T/\epsilon)+C_R^2)^{1/2}+C_R^4\right) \frac{1}{\log(1+1/\epsilon)}  s_T(\epsilon) ,\\ 
  \end{aligned}
  \end{equation*}
    Then, the cumulative regret upper bound from  Lemma~\ref{lemma:boi-inst-regret} is 
\begin{equation} \label{eqn:cumu-regret-bound-simple-1}
  \centering
  \begin{aligned}
     R_T\leq C_R^1+u_T(\epsilon), \ \text{where} \ u_T(\epsilon) =\sqrt{2c_T(\epsilon) T}
    \end{aligned}
  \end{equation}
\end{lemma}
\begin{proof}
By Lemma~\ref{lemma:boi-inst-regret},
\begin{equation} \label{eqn:cumu-regret-bound-simple-pf-1}
  \centering
  \begin{aligned}
 R_T\leq& 2 c_{B1} B + (c_{B\epsilon}(\epsilon,T)+B+c_B (B+\phi(0))) \sqrt{C_{\gamma}(\epsilon) T\gamma_T(\epsilon)}\\
 \leq& C_R^1 + (\sqrt{C_R^2+\log(1+T/\epsilon)}+C_R^3)\sqrt{C_{\gamma}(\epsilon)Ts_T(\epsilon)} \\
 =&C_R^1 + \sqrt{c_T(\epsilon)} \sqrt{2T}.
 \end{aligned}
  \end{equation}
\end{proof}

First, we present the proof for SE kernel. 
The proof of Theorem~\ref{cor:ego-nugget-se} is given below. 
\begin{proof}
From Lemma~\ref{lem:cumu-regret-bound-simple}, at a given $T$, the upper bound on $R_T(\epsilon)$ changes the same way as $c_T(\epsilon)$. In particular, by Lemma~\ref{lem:mig-se} we only need to consider 
$c_T(\epsilon)/C_{dl}^1$. 
Hence, in the following, we analyze how $c_T(\epsilon)$ changes with $\epsilon$.

   For simplicity, let $\eta=\frac{1}{\epsilon}$.
   We define the following shorthands for terms in $c_T(\eta)$:
  \begin{equation} \label{eqn:ego-nugget-se-pf-1}
  \centering
  \begin{aligned}
       c_T^1(\eta) = \log(1+T\eta)+2C_R^3 (\log(1+T\eta)+C_R^2)^{1/2}+C_R^4, 
     \end{aligned}
  \end{equation}
     and 
    \begin{equation} \label{eqn:ego-nugget-se-pf-2}
  \centering
  \begin{aligned}
       c_T^2(\eta) = \log^{d+1}(1+T\eta)
           +C^{2}_{dl} \log(1+T\eta)+C^{3}_{dl}.
     \end{aligned}
  \end{equation} 
   Thus, by Lemma~\ref{lem:mig-se},
   \begin{equation} \label{eqn:ego-nugget-se-pf-3}
  \centering
  \begin{aligned}
     c_T(\eta)/C_{dl}^1 =    c_T^1(\eta) \frac{1}{\log(1+\eta)}
         c_T^2(\eta).
      \end{aligned}
  \end{equation} 
     Define $c_T^3(\eta):=c_T^1(\eta)c_T^2(\eta)$. We expand it into the sum of nine terms:
    \begin{equation} \label{eqn:ego-nugget-se-pf-4}
  \centering
  \begin{aligned}
        &\log^{d+2}(1+T\eta)+2C_R^3 (\log(1+T\eta)+C_R^2)^{1/2}\log^{d+1}(1+T\eta)+C_R^4\log^{d+1}(1+T\eta) +\\ 
     &C_{dl}^2\log^2(1+T\eta) + 2C_R^3 C_{dl}^2(\log(1+T\eta)+C_R^2)^{1/2}\log(1+T\eta)+C_R^4 C_{dl}^2\log(1+T\eta)+\\
     &C_{dl}^3\log(1+T\eta)+2C_R^3 C_{dl}^3(\log(1+T\eta)+C_R^2)^{1/2}+C_R^4C_{dl}^3. 
     \end{aligned}
  \end{equation} 
 By~\eqref{eqn:ego-nugget-se-pf-3}, differentiating $c_T(\eta)/C_{dl}^1$ gives
  \begin{equation} \label{eqn:ego-nugget-se-pf-5}
  \centering
  \begin{aligned} 
  \frac{d c_T(\eta)}{d\eta}\frac{1}{C_{dl}^1} = \frac{1}{\log^2(1+\eta)} \left[\log(1+\eta) \frac{dc_T^3(\eta)}{d \eta} - \frac{1}{1+\eta} c_T^3(\eta)\right] . 
    \end{aligned}
  \end{equation}
     Thus,  the sign of $\frac{d c_T(\eta)}{d\eta}$ is the sign of $d_c$ defined as
      \begin{equation} \label{eqn:ego-nugget-se-pf-6}
  \centering
  \begin{aligned}
     d_c(\eta) :=\log(1+\eta) \frac{dc_T^3(\eta)}{d \eta} - \frac{1}{1+\eta} c_T^3(\eta).
     \end{aligned}
  \end{equation}
    For simplicity, we suppress the dependence on the given $T$. 
     Next, we expand $d_c(\eta)$ in~\eqref{eqn:ego-nugget-se-pf-6} as the sum of nine terms, each corresponding to one term in~\eqref{eqn:ego-nugget-se-pf-4}. 
   The first term is $\log^{d+2} (1+T\eta)$ and its contribution to $d_c(\eta)$ is
   \begin{equation} \label{eqn:ego-nugget-se-pf-7}
  \centering
  \begin{aligned}
  d_c^{1}(\eta) := \log^{d+1}(1+T\eta)\left[(d+2) \frac{1}{1/T+\eta}\log(1+\eta) - \frac{1}{1+\eta}\log(1+T\eta)\right]. 
 \end{aligned}
  \end{equation}
 If~\eqref{eqn:ego-nugget-se-1} is satisfied, then 
\begin{equation} \label{eqn:ego-nugget-se-pf-8}
  \centering
  \begin{aligned}
  d_c^1(\eta) > \frac{1}{1+\eta}\log^{d+1}(1+T\eta)((d+2)\log(1+\eta)-\log(1+T\eta) ) >\frac{1}{1+\eta}\frac{1}{d+1}\log^{d+2}(1+T\eta). 
  \end{aligned}
  \end{equation}
 Now we repeat this derivation for the second term. Its contribution to~\eqref{eqn:ego-nugget-se-pf-6} is  
  \begin{equation} \label{eqn:ego-nugget-se-pf-9}
  \centering
  \begin{aligned}
  d_c^{2}(\eta) :=& 2C_R^3 (\log(1+T\eta)+C_R^2)^{1/2}\log^{d}(1+T\eta)\left[(d+1)\log(1+\eta)\frac{1}{1/T+\eta}\right.\\
  &\left.+\frac{1}{2(1/T+\eta)}\frac{\log(1+\eta)}{1+C_R^2/\log(1+T\eta)}-\frac{1}{1+\eta}\log(1+T\eta)\right].
  \end{aligned}
  \end{equation}
  If~\eqref{eqn:ego-nugget-se-1} is satisfied, 
  \begin{equation} \label{eqn:ego-nugget-se-pf-10}
  \centering
  \begin{aligned}
   d_c^{2}(\eta)>2C_R^3 (\log(1+T\eta)+C_R^2)^{1/2}\log^{d}(1+T\eta)\frac{1}{(1+\eta)(2(d+1))}\frac{\log(1+T\eta)}{1+C_R^2/\log(1+T\eta)} .
  \end{aligned}
  \end{equation}
  The third term
  adds to~\eqref{eqn:ego-nugget-se-pf-6} as 
  \begin{equation} \label{eqn:ego-nugget-se-pf-11}
  \centering
  \begin{aligned}
  d_c^{3}(\eta) :=C_R^4\log^d(1+T\eta)\left[(d+1)\log(1+\eta)\frac{1}{1/T+\eta}-\frac{\log(1+T\eta)}{1+\eta}\right].
  \end{aligned}
  \end{equation}
  If~\eqref{eqn:ego-nugget-se-1} holds, then $d_c^{3}(\eta)>0$. 
  The fourth derivative is 
  \begin{equation} \label{eqn:ego-nugget-se-pf-12}
  \centering
  \begin{aligned}
   d_c^4(\eta) := C_{dl}^2\log(1+T\eta)\left[2\log(1+\eta)\frac{1}{1/T+\eta}-\frac{\log(1+T\eta)}{1+\eta}\right].
   \end{aligned}
  \end{equation}
  With~\eqref{eqn:ego-nugget-se-1}, we can write
  \begin{equation} \label{eqn:ego-nugget-se-pf-13}
  \centering
  \begin{aligned}
   d_c^4(\eta) > -C_{dl}^2\log^2(1+T\eta)\frac{1}{1+\eta}\frac{d-1}{d+1}.
  \end{aligned}
  \end{equation}
  For the fifth term, its role in~\eqref{eqn:ego-nugget-se-pf-6} is 
   \begin{equation} \label{eqn:ego-nugget-se-pf-14}
  \centering
  \begin{aligned}
  d_c^5(\eta) :=&2 C_R^3 C_{dl}^2 (\log(1+T\eta)+C_R^2)^{1/2}\left[\log(1+\eta)\frac{1}{1/T+\eta}+\frac{1}{2(1/T+\eta)}\frac{\log(1+\eta)}{1+C_R^2/\log(1+T\eta)}\right. \\
  & \left. -\frac{\log(1+T\eta)}{1+\eta}\right].
  \end{aligned}
  \end{equation}
  Given~\eqref{eqn:ego-nugget-se-1}, 
  \begin{equation} \label{eqn:ego-nugget-se-pf-16}
  \centering
  \begin{aligned}
  d_c^5(\eta) > \frac{2 C_R^3 C_{dl}^2}{1+\eta}(\log(1+T\eta)+C_R^2)^{1/2}\left[\frac{1}{2(d+1)}\frac{\log(1+T\eta)}{1+C_R^2/\log(1+T\eta)}-\frac{d}{d+1}\log(1+T\eta)\right].
  \end{aligned}
  \end{equation}
  The sixth and seventh terms can be combined can yield for~\eqref{eqn:ego-nugget-se-pf-6}
  \begin{equation} \label{eqn:ego-nugget-se-pf-15}
  \centering
  \begin{aligned}
  d_c^{67}(\eta) := (C_R^4 C_{dl}^2+ C_{dl}^3) \left[\log(1+\eta)\frac{1}{1/T+\eta}-\frac{\log(1+T\eta)}{(1+\eta}\right].
  \end{aligned}
  \end{equation}
  By~\eqref{eqn:ego-nugget-se-1}, 
  \begin{equation} \label{eqn:ego-nugget-se-pf-16}
  \centering
  \begin{aligned}
  d_c^{67}(\eta)>-(C_R^4 C_{dl}^2+ C_{dl}^3)\frac{1}{1+\eta} \frac{d}{d+1}\log(1+T\eta).
  \end{aligned}
  \end{equation}
  The eighth term leads to 
  \begin{equation} \label{eqn:ego-nugget-se-pf-17}
  \centering
  \begin{aligned}
  d_c^8(\eta) := 2C_R^3 C_{dl}^3(\log(1+T\eta)+C_R^2)^{1/2} \left[\frac{1}{2(1/T+\eta)}\frac{\log(1+\eta)}{\log(1+T\eta)+C_R^2}-\frac{1}{1+\eta}\right].
  \end{aligned}
  \end{equation}
  Using~\eqref{eqn:ego-nugget-se-1}, 
  \begin{equation} \label{eqn:ego-nugget-se-pf-18}
  \centering
  \begin{aligned}
  d_c^8(\eta)>2C_R^3 C_{dl}^3(\log(1+T\eta)+C_R^2)^{1/2} \frac{1}{\eta+1}[\frac{1}{2(d+1)}\frac{1}{1+C_R^2/\log(1+T\eta)}-1].
   \end{aligned}
  \end{equation}
 The last term leads to in (1) 
 \begin{equation} \label{eqn:ego-nugget-se-pf-19}
  \centering
  \begin{aligned}
   d_c^9(\eta):=-C_R^4C_{dl}^3\frac{1}{1+\eta}. 
\end{aligned}
  \end{equation}
 Since $d_c(\eta)=\sum_{i=1}^9 d_c^i(\eta)$, we consider the order of the lower bound of $d_c^i(\eta)$ to find the sign of $d_c(\eta)$. 
 Combine the lower bounds above and lift common multiplier $1/(1+\eta)$. Since  $\log(1+T\eta)\gg 1$, we can write the order of each term as follows. 
 
 The order of the positive lower bounds are $\log^{d+2}(1+T\eta)$~\eqref{eqn:ego-nugget-se-pf-8} and $\frac{1}{d}C_R^3\log^{d+1.5}(1+T\eta)$~\eqref{eqn:ego-nugget-se-pf-10}. For negative lower bounds, we have $-C_{dl}^2 \frac{d-1}{d+1}\log^2(1+T\eta)$~\eqref{eqn:ego-nugget-se-pf-12}, $-C_R^3C_{dl}^2 \log^{1.5}(1+T\eta)$~\eqref{eqn:ego-nugget-se-pf-14}, $-(C_R^4C_{dl}^2+C_{dl}^3)\log(1+T\eta)$~\eqref{eqn:ego-nugget-se-pf-16}, $-C_R^3C_{dl}^3\log^{0.5}(1+T\eta)$~\eqref{eqn:ego-nugget-se-pf-18}, and $-C_R^4C_{dl}^3$~\eqref{eqn:ego-nugget-se-pf-19}. 
 It is easy to verify that, for $d\geq 2$, 
 $$\frac{1}{d}C_R^3\log^{d+1.5}(1+T\eta) \gg C_R^3C_{dl}^2 \log^{1.5}(1+T\eta).$$
 Further, using the positive bound from~\eqref{eqn:ego-nugget-se-pf-8} satisfies
 \begin{equation} \label{eqn:ego-nugget-se-pf-20}
  \centering
  \begin{aligned}
 \log^{d+2}(1+T\eta) \gg& C_{dl}^2 \frac{d-1}{d+1}\log^2(1+T\eta),\\
  \log^{d+2}(1+T\eta) \gg &(C_R^4C_{dl}^2+C_{dl}^3)\log(1+T\eta),\\
 \log^{d+2}(1+T\eta) \gg& C_R^3C_{dl}^3\log^{0.5}(1+T\eta),\\
  \log^{d+2}(1+T\eta) \gg& C_R^4C_{dl}^3.
 \end{aligned}
  \end{equation}
 
 Thus, by~\eqref{eqn:ego-nugget-se-pf-5} and~\eqref{eqn:ego-nugget-se-pf-6}, we arrive at $d_c(\eta)>0$ and $\frac{dc_T(\eta)}{\eta}>0$.
 When $d=1$, the conclusion remains the same as $d_c^4(\eta)>0$, where $d_c^3(\eta)$ and  $d_c^4(\eta)$ coincide.

 Next, we consider if~\eqref{eqn:ego-nugget-se-2} is true. From $d_c^1(\eta)$~\eqref{eqn:ego-nugget-se-pf-7} and~\eqref{eqn:ego-nugget-se-2}, we have $d_c^1(\eta)<0$. Indeed, it is easy to verify that all $d_c^i(\eta)<0, i=2,..., 9$. Hence, $d_c<0$  and $\frac{dc_T(\eta)}{\eta}<0$.
  \end{proof}

The proof of Theorem~\ref{cor:ego-regret-nugget-matern} is given below.
\begin{proof} 
Recall that from Lemma~\ref{lem:mig-matern-shogo} that $s_T(\epsilon)= c_T^0(\epsilon)\bar{\gamma}_T(\epsilon)+C$.
By Lemma~\ref{lem:cumu-regret-bound-simple}, the cumulative regret bound changes the same way as $c_T(\epsilon)/(C_{d\nu l}^1C_{\nu}^1)$. 
    First, define 
    \begin{equation} \label{eqn:ego-nugget-matern-pf-1}
  \centering
  \begin{aligned}
    c_T^1(\eta) :=& \log(1+T\eta)+2C_R^3 (\log(1+T\eta)+C_R^2)^{1/2}+C_R^4,\\
    c_T^2(\eta) :=&c_T^0(\eta)/  C_{\nu}^1=[ \log(
            1+C_{\nu}^2 \log\left(T^2 \eta\right))
          +C_{\nu}^3 \log (T^2\eta)+C_{\nu}^4],\\
          c_T^3(\eta) :=& \bar\gamma_T(\eta)/ C^{1}_{d\nu l}=[
        \log(1+2T\eta)+  C^{2}_{d\nu l}(T\eta)^{\frac{d}{2\nu+d}} \log^{\frac{2\nu}{2\nu+d}}(1+2T\eta)].
    \end{aligned}
  \end{equation}
    Let $c_T^{123}(\eta) =c_T^1(\eta)c_T^2(\eta)c_T^3(\eta)$.
    By Lemma~\ref{lem:mig-matern-shogo}, we have 
    \begin{equation} \label{eqn:ego-nugget-matern-pf-2}
  \centering
  \begin{aligned}
    c_T(\eta) = C_{\nu}^1C_{d\nu l}^1c_T^{123}(\eta)\frac{1}{\log(1+\eta)}+c_T^1(\eta)\frac{1}{\log(1+\eta)}C. 
    \end{aligned}
  \end{equation}
   Consider the first part of~\eqref{eqn:ego-nugget-matern-pf-2}. Its derivative is
  \begin{equation} \label{eqn:ego-nugget-matern-pf-3}
  \centering
  \begin{aligned}
 C_{\nu}^1C_{d\nu l}^1 \frac{1}{\log^2(1+\eta)} \left[\log(1+\eta) \frac{d  c_T^{123}(\eta) }{d \eta} - \frac{1}{1+\eta}  c_T^{123}(\eta)\right]  . 
 \end{aligned}
  \end{equation}
     The sign of~\eqref{eqn:ego-nugget-matern-pf-3} is the same as the sign of 
     \begin{equation} \label{eqn:ego-nugget-matern-pf-4}
  \centering
  \begin{aligned}
    d_c(\eta) :=\frac{d  c_T^{123}(\eta)}{d \eta} \log(1+\eta)- \frac{1}{1+\eta}  c_T^{123}(\eta) .
     \end{aligned}
  \end{equation}
   Consider case (1) first. Since $c_T^i(\eta)>0$ and increases with $\eta$, $i=1,2,3$, we have 
     \begin{equation} \label{eqn:ego-nugget-matern-pf-5}
  \centering
  \begin{aligned}
     \frac{d  c_T^{123}(\eta) }{d \eta} =& \frac{d c_T^1(\eta)}{d\eta} c_T^2(\eta)c_T^3(\eta)+\frac{d c_T^2(\eta)}{d\eta} c_T^1(\eta) c_T^3(\eta)+\frac{d c_T^3(\eta)}{d\eta} c_T^1(\eta)c_T^2(\eta)\\
     >& \frac{d c_T^3(\eta)}{d\eta} c_T^1(\eta)c_T^2(\eta).
     \end{aligned}
  \end{equation}
    The first two terms in the first equality above will be used later. 
    Using~\eqref{eqn:ego-nugget-matern-pf-4} in~\eqref{eqn:ego-nugget-matern-pf-5}, we have the lower bound  
     \begin{equation} \label{eqn:ego-nugget-matern-pf-6}
  \centering
  \begin{aligned}
    d_c(\eta) > c_T^1(\eta)c_T^2(\eta)\left[ \frac{dc_T^3(\eta)}{d\eta}\log(1+\eta)-\frac{1}{1+\eta}c_T^3(\eta)\right]. 
     \end{aligned}
  \end{equation}
     The derivative is
     \begin{equation} \label{eqn:ego-nugget-matern-pf-7}
  \centering
  \begin{aligned}
     \frac{dc_T^3(\eta)}{d\eta} = \frac{1}{\eta+1/2T}+ C_{d\nu l}^2 \frac{d}{2\nu+d}(T\eta)^{\frac{d}{2\nu+d}}\frac{1}{\eta}\log^{\frac{2\nu}{2\nu+d}}(1+2T\eta) +\\
 C^{2}_{d\nu l}\left(T\eta\right)^{\frac{d}{2\nu+d}} \frac{2\nu}{2\nu+d}\log^{-\frac{d}{2\nu+d}}(1+2T\eta) \frac{1}{\eta+1/2T}.
 \end{aligned}
  \end{equation}
     Note that the first and third terms in~\eqref{eqn:ego-nugget-matern-pf-7} $>0$. 
     By relaxing them and using~\eqref{eqn:ego-nugget-matern-pf-7} in~\eqref{eqn:ego-nugget-matern-pf-6}, we have
      \begin{equation} \label{eqn:ego-nugget-matern-pf-8}
  \centering
  \begin{aligned}
    d_c(\eta) > c_T^1(\eta)c_T^2(\eta)C_{d\nu l}^2 (T\eta)^{\frac{d}{2\nu+d}}\log^{\frac{2\nu}{2\nu+d}}(1+2T\eta)\frac{1}{1+\eta} \left[ \frac{d}{2\nu+d} \log(1+\eta) \right. \\
    \left. - \left(\frac{1}{C_{d\nu l}^2} \left(\frac{\log(1+2T\eta)}{T\eta}\right)^{\frac{d}{2\nu+d}}+1\right)\right].
    \end{aligned}
  \end{equation}
          Since $\frac{\log(1+2T\eta)}{T\eta}<1$ for $T\eta\gg1$, from~\eqref{eqn:ego-nugget-matern-1}, we have 
     \begin{equation} \label{eqn:ego-nugget-matern-pf-9}
  \centering
  \begin{aligned}
    d_c(\eta)> c_T^1(\eta)c_T^2(\eta)C_{d\nu l}^2 (T\eta)^{\frac{d}{2\nu+d}}\log^{\frac{2\nu}{2\nu+d}}(1+2T\eta) \frac{1}{1+\eta}\left[ \frac{d\log(1+\eta)}{2\nu+d}  -  (1/C_{d\nu l}^2 +1)\right]>0.
     \end{aligned}
  \end{equation}
     
 Next, we consider the second part in~\eqref{eqn:ego-nugget-matern-pf-2}.  Its derivative is
 \begin{equation} \label{eqn:ego-nugget-matern-pf-10}
  \centering
  \begin{aligned}
    \frac{C}{\log^2(1+\eta)} \left[\log(1+\eta) \frac{d  c_T^{1}(\eta) }{d \eta} - \frac{1}{1+\eta}  c_T^{1}(\eta )\right].   
 \end{aligned}
  \end{equation}
 Given~\eqref{eqn:ego-nugget-matern-1}, we know the second term in~\eqref{eqn:ego-nugget-matern-pf-5} leads to 
 \begin{equation} \label{eqn:ego-nugget-matern-pf-11}
  \centering
  \begin{aligned}
     C_{\nu}^1C_{d\nu l}^1\frac{d c_T^2(\eta)}{d\eta} c_T^1(\eta)c_T^3(\eta)\log(1+\eta)
     >& C_{\nu}^1C_{d\nu l}^1C_{\nu}^3 \frac{1}{\eta} c_T^1(\eta)\log(1+2T\eta)\log(1+\eta)
     >c_T^1(\eta)\frac{C}{1+\eta} .
     \end{aligned}
  \end{equation}
 Hence, by~\eqref{eqn:ego-nugget-matern-pf-3},~\eqref{eqn:ego-nugget-matern-pf-9},~\eqref{eqn:ego-nugget-matern-pf-2} and~\eqref{eqn:ego-nugget-matern-pf-11}, we have $\frac{d c_T(\eta)}{d\eta} >0$.
    
     Next, we consider case 2. We use a more compact proof procedure given that $c_T^{123}(\eta)$ has $18$ terms. Define $ c_T^{13}(\eta):=c_T^1(\eta) c_T^3(\eta)=\sum_{p=1}^3\sum_{q=1}^2 a^1_p(\eta) a_q^3(\eta)$, where 
     \begin{equation} \label{eqn:ego-nugget-matern-pf-12}
  \centering
  \begin{aligned}
     a^1_1(\eta)=&
     \log(1+T\eta), a^1_2(\eta) = 2C_R^3 (\log(1+T\eta)+C_R^2)^{1/2}, a^1_3(\eta) = C_R^4, \\
     a^3_1(\eta)=& \log(1+2T\eta),
     a^3_2(\eta) = C^{2}_{d\nu l}(T\eta)^{\frac{d}{2\nu+d}} \log^{\frac{2\nu}{2\nu+d}}(1+2T\eta).
     \end{aligned}
  \end{equation}
     Thus, $d_c(\eta)$ can be written as   
     \begin{equation} \label{eqn:ego-nugget-matern-pf-13}
  \centering
  \begin{aligned}
     d_c(\eta)=\sum_{p=1}^3 \sum_{q=1}^2 c_T^2(\eta)a^1_p(\eta)a^3_q(\eta) &\left[  \left(\frac{da^1_p(\eta)}{d\eta} /a^1_p(\eta)+\frac{da^3_q(\eta)}{d\eta}/a^3_q(\eta)+\frac{d c_T^2(\eta)}{d\eta}/c_T^2(\eta)\right)\log(1+\eta)-\frac{1}{1+\eta}\right] .
     \end{aligned}
  \end{equation}
  We note that $c_T^2(\eta)$ is not decomposed for a even more compact proof because it is possible that $C_{\nu}^4<0$. 
     Since $a^1_p(\eta),a^3_q(\eta),c_T^2(\eta)>0$, to show $d_c(\eta)<0$, it is sufficient to find the largest value of 
     \begin{equation} \label{eqn:ego-nugget-matern-pf-14}
  \centering
  \begin{aligned}
      d_{pq\eta} := \left[\frac{da^1_p(\eta)}{d\eta}/a^1_p(\eta)+\frac{da^3_q(\eta)}{d\eta}/a^3_q(\eta)+\frac{d c_T^2(\eta)}{d\eta}/c_T^2(\eta)\right]\log(1+\eta)-\frac{1}{1+\eta}, 
      \end{aligned}
  \end{equation}
      for all $p$ and $q$, a total of six terms, and show that it is $<0$. 
     If this largest term generates $d_{pq\eta}<0$, then the other five $p,q$ combinations must also have $d_{pq\eta}<0$ and, hence, the sum~\eqref{eqn:ego-nugget-matern-pf-13} will be $<0$. 
     Using elementary algebra one can compare the six terms and verify that the largest $ d_{pq\eta}(\eta)$ occurs at $p=1$ and $q=2$ given $T\eta\gg 1$.  
     The third term in~\eqref{eqn:ego-nugget-matern-pf-14} is:  
     \begin{equation} \label{eqn:ego-nugget-matern-pf-15}
  \centering
  \begin{aligned}
      \frac{dc_T^2(\eta)}{d\eta}/c_T^2(\eta) < \frac{1}{\eta} \frac{ \frac{1}{1/C_{\nu}^2+\log(T^2\eta)}+C_{\nu}^3}{c_T^2(\eta)}<\frac{1}{\eta}\frac{1}{c_T^2(\eta)}\left(\frac{1}{\log(T^2\eta)}+C_{\nu}^3\right).
      \end{aligned}
  \end{equation}
     Further, 
     \begin{equation} \label{eqn:ego-nugget-matern-pf-16}
  \centering
  \begin{aligned}
     \frac{da^1_1(\eta)}{d\eta}/a^1_1(\eta)= \frac{1}{1/T+\eta}\frac{1}{\log(1+T\eta)},
     \frac{da^3_2(\eta)}{d\eta}/a^3_2(\eta)< \frac{1}{\eta}\left[\frac{d}{2\nu+d}+\frac{2\nu}{2\nu+d}\frac{1}{\log(1+2T\eta)}\right].
     \end{aligned}
  \end{equation}
     Thus, for any $p,q$, we can write
     \begin{equation} \label{eqn:ego-nugget-matern-pf-17}
  \centering
  \begin{aligned}
      d_{pq\eta} <&\frac{1}{\eta}\left[\frac{1}{\log(T\eta)}+
     \frac{d}{2\nu+d}+\frac{2\nu}{2\nu+d}\frac{1}{\log(2T\eta)}+\frac{1}{c_T^2(\eta)}\frac{1}{\log(T^2\eta)}+\frac{1}{c_T^2(\eta)}C_{\nu}^3\right]\log(1+\eta)-\frac{1}{1+\eta}\\
     <&\frac{1}{\eta}\left[\frac{d}{2\nu+d}+\frac{1}{\log(T\eta)}\left(\frac{4\nu+d}{2\nu+d}+\frac{1}{c_T^2(\eta)}\right)+\frac{C_{\nu}^3}{c_T^2(\eta)}\right]\log(1+\eta) -\frac{1}{1+\eta},
     \end{aligned}
  \end{equation}
     where we use $\log(T^2\eta)\geq \log(T\eta)$.     
     By~\eqref{eqn:ego-nugget-matern-2},~\eqref{eqn:ego-nugget-matern-pf-17} $<0$. Thus, the other five (smaller) terms in~\eqref{eqn:ego-nugget-matern-pf-14} are also $<0$. Thus, $d_c<0$ by~\eqref{eqn:ego-nugget-matern-pf-13}. Further, under~\eqref{eqn:ego-nugget-matern-2}, it is easy to verify that the second part in~\eqref{eqn:ego-nugget-matern-pf-2} $<0$ and thus $\frac{dc_T(\eta)}{d\eta}<0$. By Lemma~\ref{lem:cumu-regret-bound-simple}, the cumulative regret bound changes the same as $c_T$ with respect to $\epsilon=1/\eta$. 
\end{proof}

\begin{figure}
	\centering
	\includegraphics[width=0.7\linewidth]{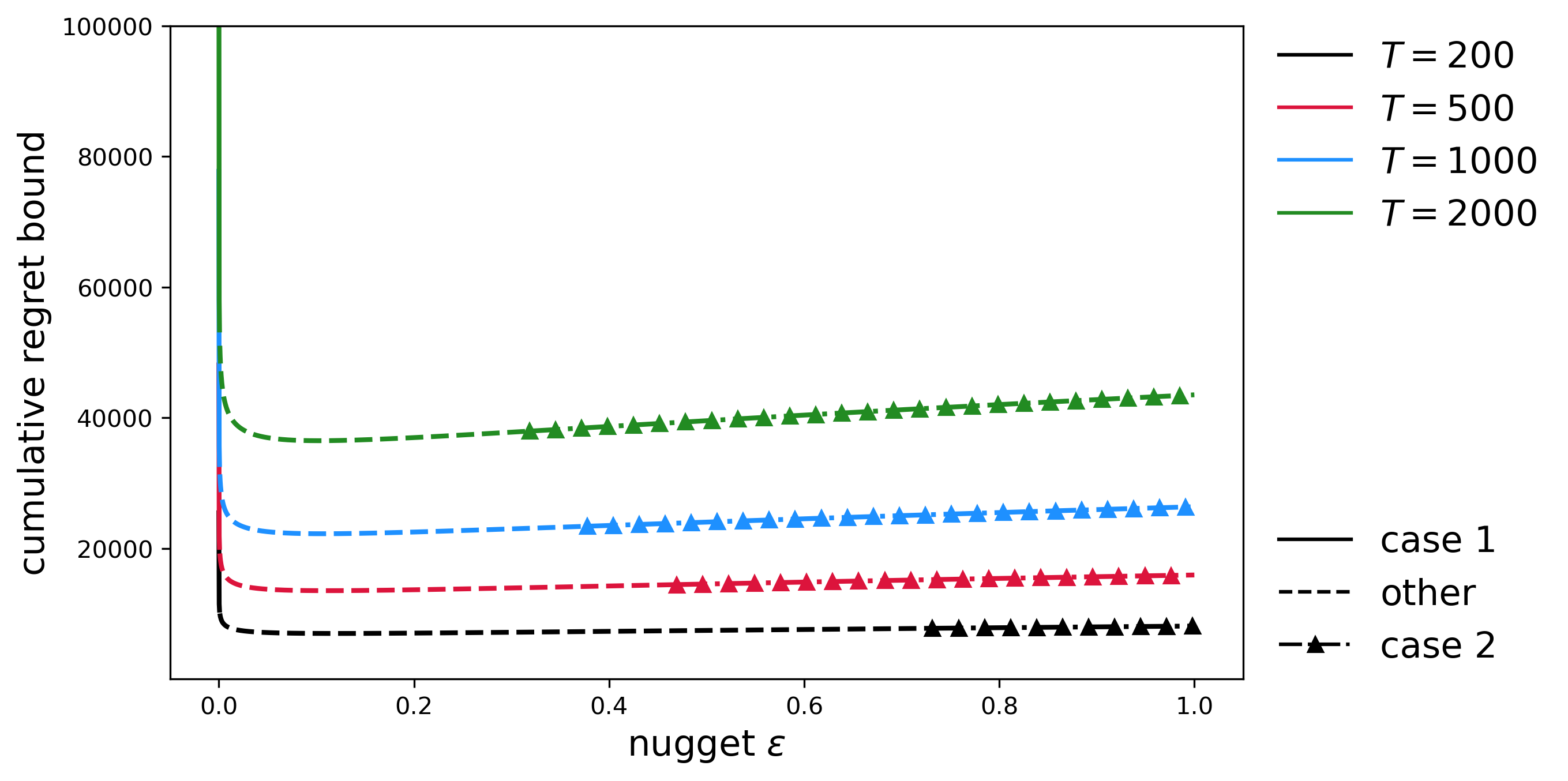}
        \caption{Cumulative regret upper bound with nugget $\epsilon$ at different $T$ and selected constants for Matérn kernel ($\nu=\frac{1}{2}$). The case ``other'' means neither the conditions for case $1$ nor those for case $2$ are satisfied.}
	\label{fig:regret-matern-nugget}
\end{figure}

\section{Numerical Example and Additional Plots}\label{appx:example}
\subsection{Plot for the Nugget Effect}
As mentioned in Section~\ref{se:nugget-EI}, we choose the $50$ samples entirely through random sampling for the Branin function and plot the EI contour using different nugget values in Figure~\ref{fig:ei-nugget-random}. 
\begin{figure}
  \centering
  \includegraphics[width=0.98\textwidth]{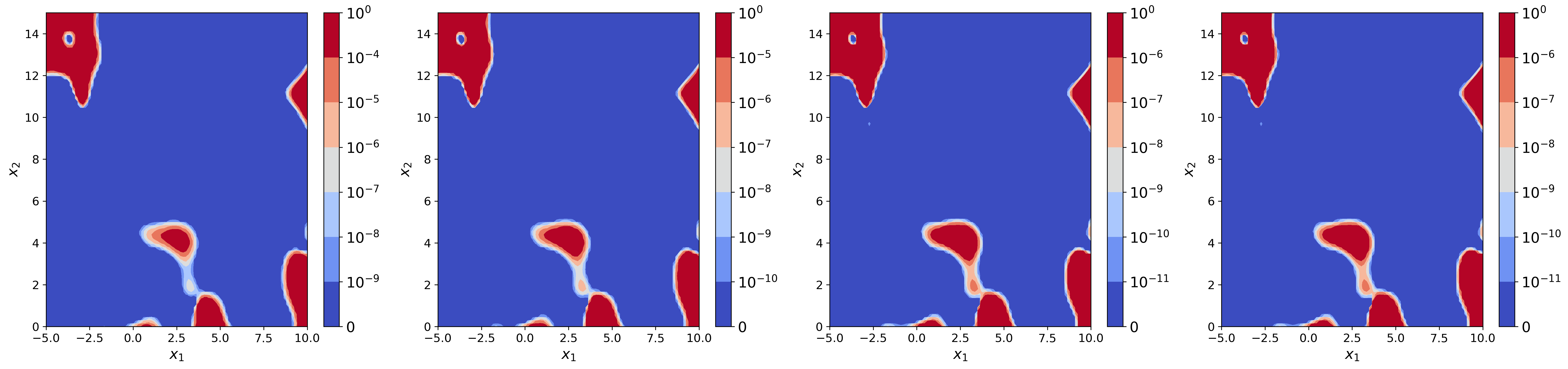}
        \caption{Illustrative example of EI contour of the Branin function with 50 random samples. From left to right: contour plots for $\epsilon=10^{-2}$, $\epsilon=10^{-6}$, $\epsilon=10^{-10}$, and no nugget. The maximum EI$_{50}$ value from left to right: 
        $0.40515$, $0.40085$, $0.40085$, and $0.40085$. The maximum level of the colorbar corresponds to the maximum of EI$_{50}$ of each plot.}
\label{fig:ei-nugget-random}
\end{figure}
The effect of $\epsilon$ on the values of EI is much less obvious here, as expected.
\subsection{Numerical Example Setup}
The mathematical expression for example 1, the two-dimensional  Rosenbrock function, is given below.
\begin{equation} \label{rosenbrock}
  \centering
  \begin{aligned}
    f(\xbm) =\sum_{i=1}^{d-1} \left[100(x_{i+1}-x_i^2)^2+(x_i-1)^2  \right]\\
    x_i \in [-2.048, 2.0480]^2.
\end{aligned}
  \end{equation}
 The optimal objective function value is $0$.

 The mathematical expression for example 2, the six-hump camel function given below.
 \begin{equation} \label{camel}
  \centering
  \begin{aligned}
    f(\xbm) =\left(4-2.1x_1^2+\frac{x_1^4}{3}  \right) x_1^2 +x_1 x_2 +(-4+4x_2^2)x_2^2\\
    -3\leq x_1 \leq 3, \ -2\leq x_2 \leq 2.
\end{aligned}
  \end{equation}
 The optimal objective function value is $-1.0316$.
 
The mathematical expression for example 3, the Hartmann6 function is given below.
     \begin{equation} \label{hartmann6}
  \centering
  \begin{aligned}
f(\xbm) = - \sum_{i=1}^{4} \alpha_i \exp \left( -\sum_{j=1}^{6} A_{ij} (x_j - P_{ij})^2 \right)\\
					x_i \in [0, 1], i = 1, \dots,6\\
					\alpha = [1.0, 1.2, 3.0, 3.2]^\top\\
					A = 
					\begin{bmatrix}
						10 & 3.0 & 17 & 3.5 & 1.7 & 8.0 \\
						0.05 & 10 & 17 & 0.1 & 8.0 & 14 \\
						3.0 & 3.5 & 1.7 & 10 & 17 & 8.0 \\
						17 & 8.0 & 0.05 & 10 & 0.1 & 14
					\end{bmatrix}\\
					P = 
					\begin{bmatrix}
						0.131 & 0.170 & 0.557 & 0.012 & 0.828 & 0.587 \\
						0.233 & 0.414 & 0.831 & 0.374 & 0.100 & 0.999 \\
						0.235 & 0.145 & 0.352 & 0.288 & 0.305 & 0.665 \\
						0.405 & 0.883 & 0.873 & 0.574 & 0.109 & 0.038
					\end{bmatrix}.\\
     \end{aligned}
  \end{equation}
  The optimal objective function value is $-3.32$.
  
The mathematical expression for example 4, the Branin function,  is given below.
\begin{equation} \label{branin}
  \centering
  \begin{aligned}
    f(\xbm) =\left( x_2-\frac{5.1}{4\pi^2}x_1^2+\frac{5}{\pi}x_1-6\right)^2 +10\left(1-\frac{1}{8\pi}\right)\cos(x_1) + 10\\
    x_1 \in [-5, 10], x_2 \in [0,15].
\end{aligned}
  \end{equation}
 The optimal objective function value is $0$.

 The mathematical expression for example 5, the Michalewicz  function given below.
 \begin{equation} \label{michalewicz}
  \centering
  \begin{aligned}
    f(\xbm) =  -\sum_{i=1}^2 \sin(x_i) \sin^{20}\left( \frac{ix_i^2}{\pi} \right) \\
    \xbm \in [0,\pi]^2.
\end{aligned}
  \end{equation}
 The optimal objective function value is $-1.8013$.

  The $25$th and $75$th percentile of the 100 repeated runs are shown in the following figure.
\begin{figure}
  \centering

  \begin{subfigure}[b]{0.45\textwidth}
    \centering
    \includegraphics[width=\linewidth]{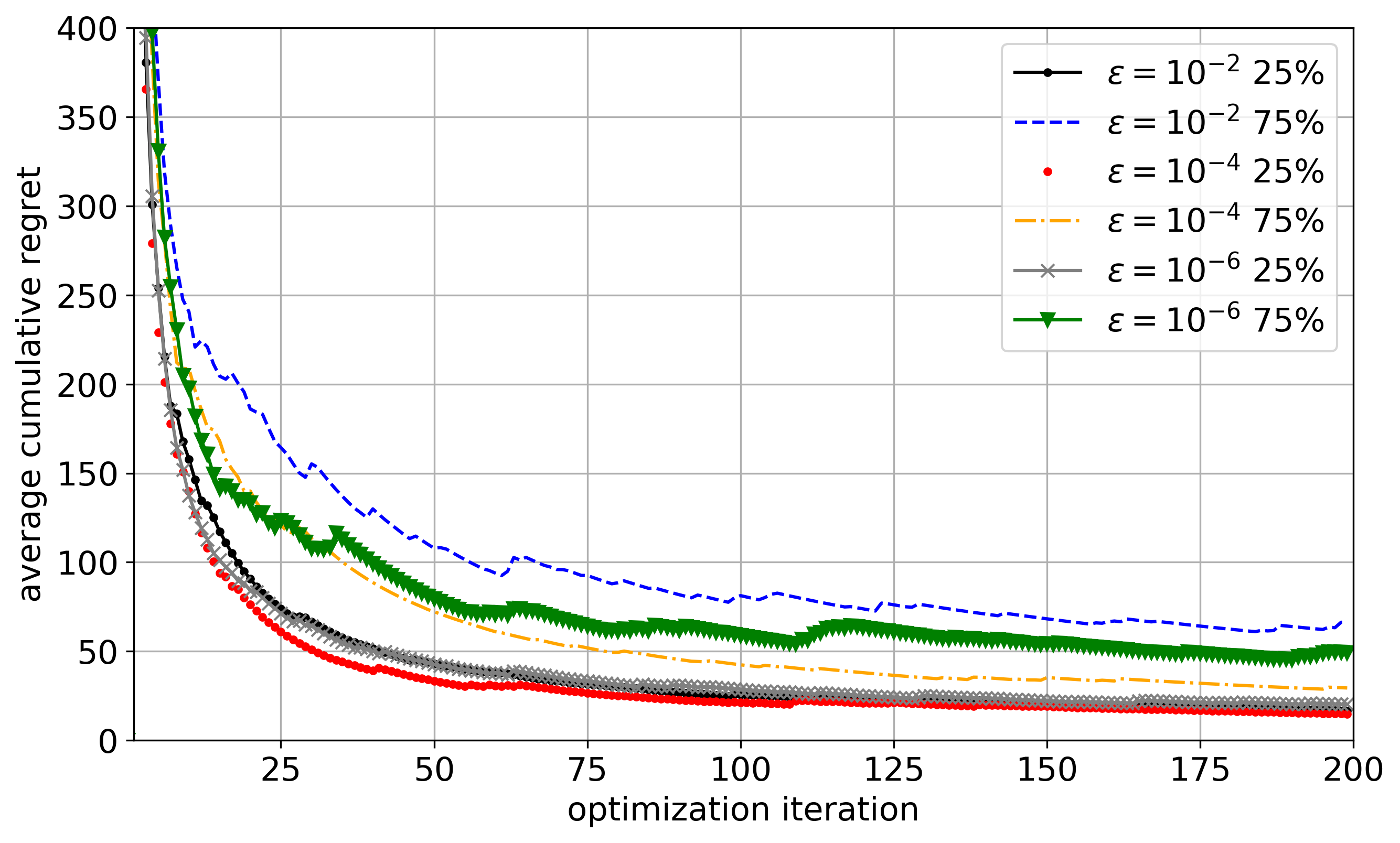}
    \caption{Example 1: Rosenbrock function}
    \label{fig:ex2_sta}
  \end{subfigure}
  \hfill
  \begin{subfigure}[b]{0.45\textwidth}
    \centering
    \includegraphics[width=\linewidth]{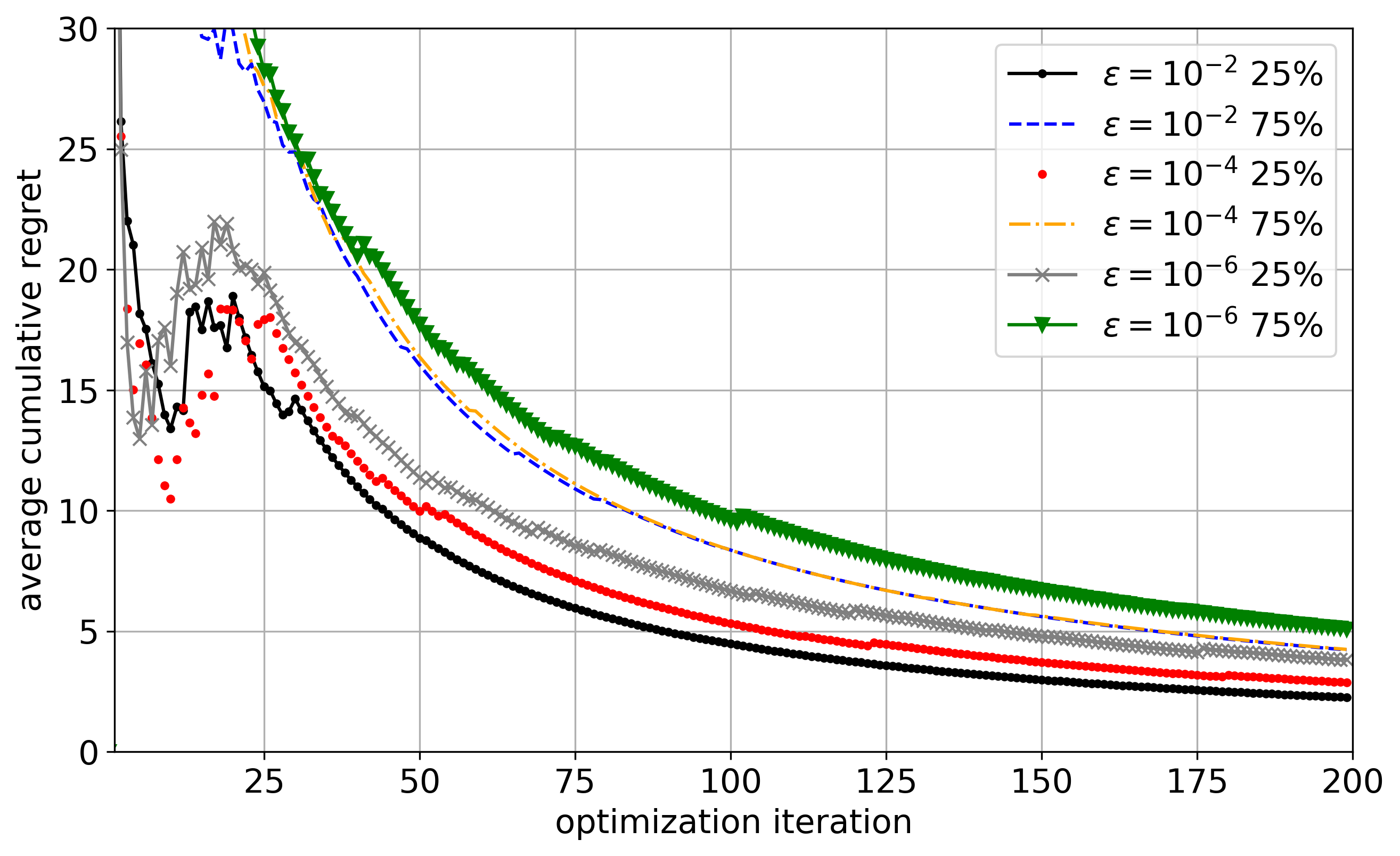}
    \caption{Example 2: Six-hump camel function}
    \label{fig:ex3_sta}
  \end{subfigure}

  \vspace{0.6em}

  \begin{subfigure}[b]{0.45\textwidth}
    \centering
    \includegraphics[width=\linewidth]{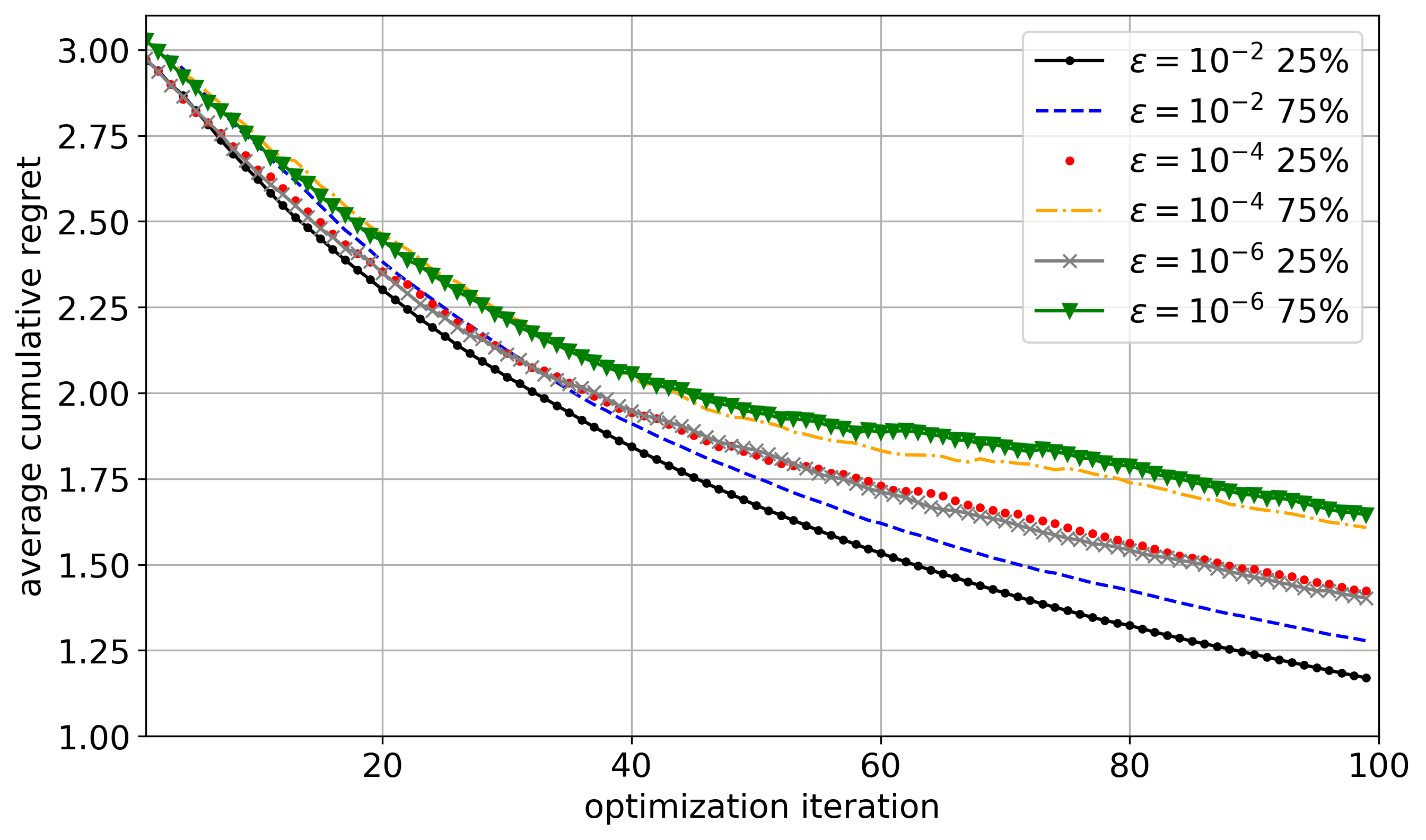}
    \caption{Example 3: Hartmann6}
    \label{fig:ex5_sta}
  \end{subfigure}
  \hfill
  \begin{subfigure}[b]{0.45\textwidth}
    \centering
    \includegraphics[width=\linewidth]{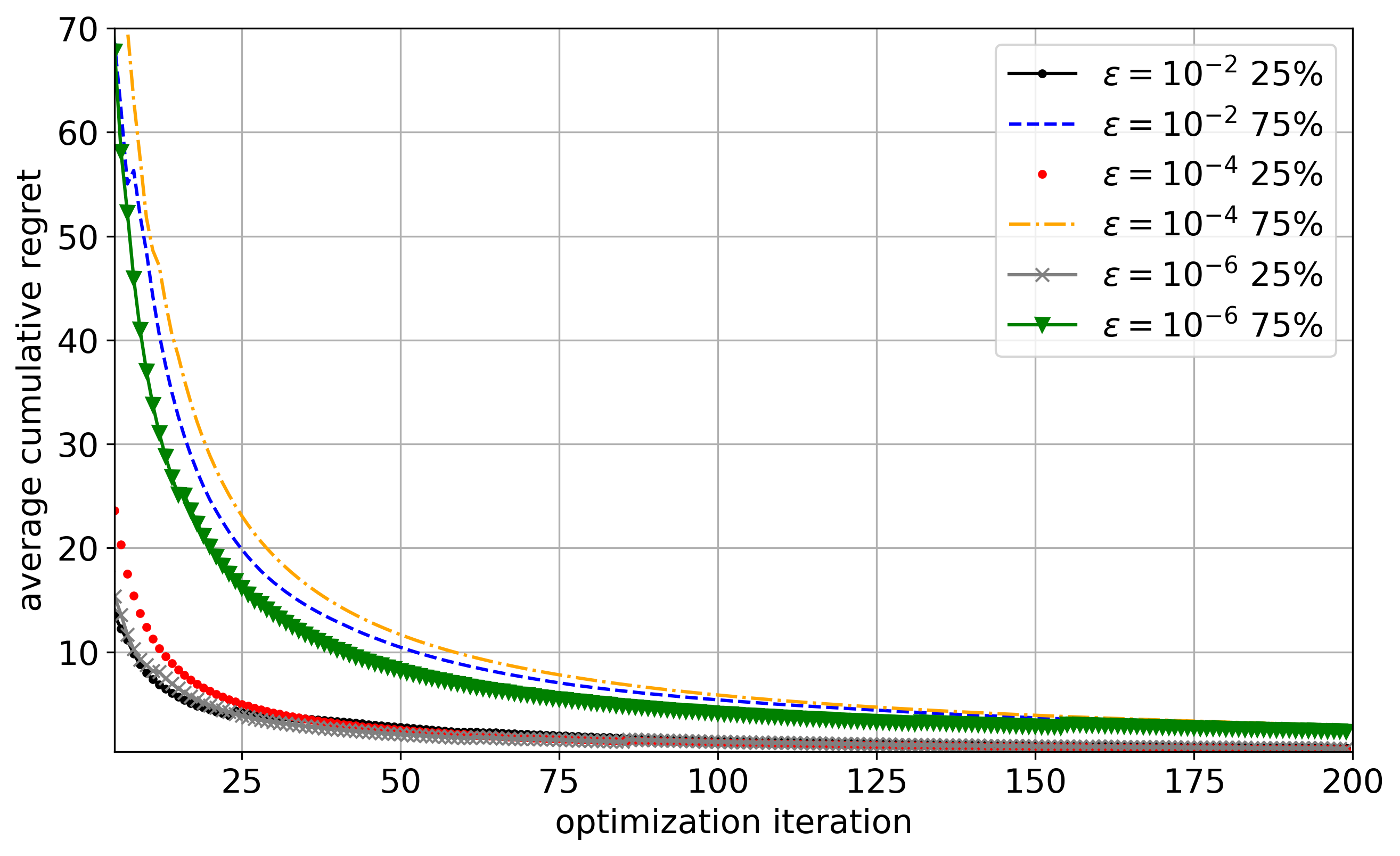}
    \caption{Example 4: Branin function}
    \label{fig:ex1_sta}
  \end{subfigure}

  \vspace{0.6em}

  \begin{subfigure}[b]{0.45\textwidth}
    \centering
    \includegraphics[width=\linewidth]{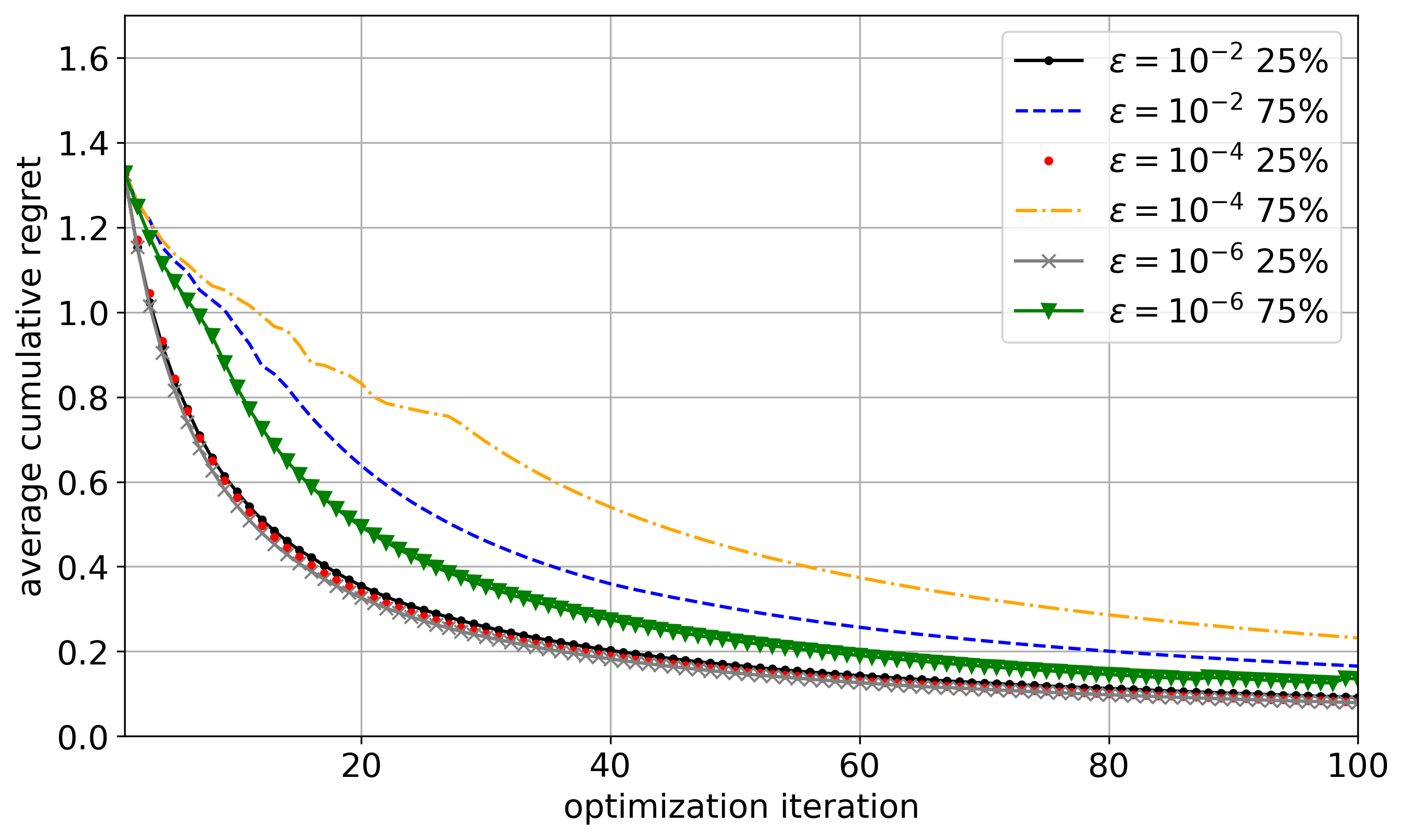}
    \caption{Example 5: Michalewicz function}
    \label{fig:ex4_sta}
  \end{subfigure}

  \caption{25th and 75th percentile average cumulative regret for practical EGO with nugget values $10^{-2}$, $10^{-4}$, and $10^{-6}$ for five examples.}
  \label{fig:numerical_sta}
\end{figure}



\clearpage
\bibliographystyle{plainnat}
\bibliography{bibliography}
\newpage

\end{document}